\documentclass[10pt]{article}

\usepackage{parskip}
\usepackage{graphicx} 
\usepackage{amsmath,amssymb,amsthm,mathtools}
\usepackage{hyperref}
\usepackage{url}
\usepackage{fullpage}
\usepackage{enumitem}
\usepackage{algorithm} 
\usepackage{algorithmic} 
\usepackage[T1]{fontenc}
\usepackage{graphicx}
\usepackage{booktabs}
\usepackage{multirow}
\usepackage{makecell}
\usepackage{natbib}
\usepackage{bigstrut}
\usepackage{subcaption}
\usepackage{xcolor}
\usepackage{bm}
\usepackage{microtype}

\newtheorem{theorem}{Theorem}[section]
\newtheorem{proposition}[theorem]{Proposition}

\DeclareMathAlphabet\mathbfcal{OMS}{cmsy}{b}{n}

\title{Beyond the Matrix Sign: Quadratic Spectral Descent}
\author{Qiaozhe Zhang, Jun Sun, Yingzhuang Liu\\
School of Electronic Information and Communications\\
Huazhong University of Science and Technology\\
\texttt{\{qiaozhezhang, juns, liuyz\}@hust.edu.cn}}
\date{September 2026}

\begin{document}

\maketitle

\begin{abstract}
Muon emerges as a strong competitor of the AdamW for LLM pretraining, because the matrix-wise update it employs can potentially incur smaller second-order penalty than the once dominating AdamW, which performs coordinate-wise update. However, the spectral flattening procedure in Muon is quite debatable since it discards the spectral amplitude information totally. To seek  for better spectral  allocation (and the associated spectral subspace), we propose to solve the quadratic model of loss function under  the spectral norm constraint \textit{directly} (i.e., in a genuinely Newtonian way) and thus obtaining the Quadratic Spectral Descent (QSD) algorithm. In contrast, many existing curvature-aware methods either exploit the second-order information in an \textit{implicit} way by changing the weight update geometry (such as Mousse, FISMO) or rely on strong assumptions (such as the weight displacement isotropy assumption in Newton-Muon). QSD's potential advantage over these methods is best illustrated in the isotropic curvature scenario, where Mousse, FISMO and Newton-Muon all reduce to Muon while the spectral allocation in QSD is still  \textit{non-flat} (since the spectral allocation in QSD depends on the \textit{gradient to curvature ratio}). Meanwhile, to control the complexity of QSD, we employ inversion-free K-FAC and \textit{online} Frank-Wolfe update which is essentially a matrix sign operator. Overall, the complexity increase can be rather mild. Experiments on GPT pre-training show that QSD consistently improves validation loss over Muon and recent Muon variants, while achieving up to an $8.49\%$ wall-clock speedup at matched validation loss.
\end{abstract}

\vspace{-0.9em}
\section{Introduction}
\vspace{-0.5em}

AdamW \citep{loshchilov2017decoupled} is the de-facto standard optimizer for large-scale neural networks. 
However, its role has  been  challenged recently by Muon \citep{jordan2024muon,liu2025muon,shah2025practical,wen2026fantastic}, which distinguishes itself as a matrix-valued optimizer with spectral norm constraint, as opposed to the conventional vector-valued optimizers constrained by the Euclidean norm. By employing the matrix-wise update (specifically, matrix orthogonalization and spectral flattening), Muon has the potential advantage over AdamW which uses coordinate-wise update in that: Muon can incur less second-order penalty than AdamW while the first-order gain of them two are similar, as reported in \cite{wang2026muon}. Thus Muon can achieve larger loss decrease per step compared with AdamW overall. 

Muon can be proved to be the best optimizer for the \textit{linear} model under the spectral norm constraint. In specific, in an idealized form, Muon computes \citep{pethick2025training,bernstein2024old}
\begin{equation}
D_{\mathrm{G}}
=
-\rho\,\operatorname{msgn}(G)
=
-\rho\,UV^\top
=
\arg\min_{\|D\|_2\leq\rho}
\langle G,D\rangle,
\label{eq:intro_muon}
\end{equation}
where $G = U\Sigma V^\top$. Meanwhile, Muon indeed demonstrates superior performance than AdamW  in practical LLM pretraining \citep{liu2025muon}.

Despite of the above, one of the key operation of Muon, i.e. spectral flattening, is still quite debatable even questionable, since by that operation all the spectrum amplitudes information is discarded. A natural question is thus: \textit{can we find a more efficient way for spectral allocation (as well as the associated spectral subspace update), if we are provided some kind of second-order information?}

To answer this question, several curvature-aware Muon variants have been proposed, including Newton-Muon \citep{du2026newton}, Mousse  \citep{zhang2026mousse}, FISMO  \citep{xu2026fismo} etc. Newton-Muon  \citep{du2026newton} is based on an important yet strong assumption, i.e., the weight displacement isotropy, thus it can remove the effect of the output curvature and finally obtain a right-preconditioned matrix sign solver. Mousse \citep{zhang2026mousse}  employs the Shampoo-style second-order information \citep{gupta2018shampoo,morwani2025new}(namely gradient covariance) to whiten  the weight update. FISMO \citep{xu2026fismo} instead employs the K-FAC-based \citep{martens2015optimizing}  Fisher information matrix (FIM) to modify the local geometry of the weight update (i.e., natural gradient trust region) . All these second-order Muon variants are reported to outperform the vanilla Muon. However, it can be seen that all of them either rely on strong assumption (Newton-Muon) or exploit the second-order information in an \textit{implicit} way (Mousse, FISMO) by changing the local weight geometry  and thus resulting in a preconditioned matrix sign solution. When the curvature is isotropic, unfortunately, it turns out that all the above Muon variants  will reduce to Muon, i.e., the spectrum will still be flat, thus at the risk of performance degradation \citep{su2025isotropic}.

To fully tackle the above question, in this paper we will take a different yet more direct route. Specifically, we will solve the quadratic model under the spectral-norm constraint  \textit{directly} in the \textit{general} form (i.e., without any assumptions), thus obtaining the Quadratic Spectral Descent (QSD) algorithm. Concretely, QSD aims to solve:
\begin{equation}
\min_{\|D\|_2\leq\rho}
\left\{
\langle G,D\rangle
+
\frac{\eta}{2}
\operatorname{vec}(D)^\top C\,\operatorname{vec}(D)
\right\},
\label{eq:intro_qsd}
\end{equation}
where $C\succeq 0$ models the local curvature. 
It is worth noting that for this quadratic objective optimization, the solution is $s_i^\star = \min
\left\{
\rho,
\frac{g_i}{\eta\kappa}
\right\}$ in the isotropic-curvature and aligned-spectrum (i.e., gradient and curvature) case, where $s_i^\star$ and $g_i$ denote the spectral allocation and  gradient strength of the $i$-th mode, respectively, and $\kappa$ is the isotropic curvature, implying non-flat spectral allocation generally. This is in stark contrast to the above-discussed second-order Muon variants. The differences among them are summarized in
Table~\ref{tab:optimizer_comparison}.

Furthermore, to make the quadratic spectral problem practical at modern training scale, QSD employs K-FAC \citep{martens2015optimizing,george2018fast} only to obtain structured curvature with no matrix inversion involved. Then  we solve the constrained quadratic by \textit{online} Frank-Wolfe updates \citep{frank1956algorithm,jaggi2013revisiting}. The total complexity  increase of QSD relative to Muon turns out to be rather mild due to two important reasons: 1) Online optimization: we employ the warm-start over each outer iteration, thus offering a high-quality initial value for the inner Frank-Wolfe updates; 2) Closed-form Frank-Wolfe solution: each Frank-Wolfe
step  has a closed-form solution, namely the matrix sign operator as in Muon.
Together, they make it possible that 3 Frank-Wolfe steps for each outer iteration are sufficient, thus substantially lowering the total complexity increase.


Overall, our analysis shows that the quadratic model changes the  solution of linear model for Muon significantly, in terms of both  spectral amplitude and spectral directions. Specifically, when the gradient and curvature are aligned in terms of  singular directions, the spectral amplitudes should be proportional to the\textit{ gradient-to-curvature ratio}, rather than being uniform (as in Muon).
Moreover, if they are not aligned, the optimal singular directions also need rotation relative to those of Muon. Experiments on language-model pretraining show that QSD achieves these curvature-aware updates with mild  computation overhead and consistent performance improvement over Muon.

\begin{table*}[t]
\centering
\caption{Comparison of spectral matrix optimizers. Here, $\operatorname{Polar}(X)=UV^\top$ for the singular value decomposition $X=U\Sigma V^\top$.}
\vspace{-0.9em}
\label{tab:optimizer_comparison}

\setlength{\tabcolsep}{5pt}
\renewcommand{\arraystretch}{1.25}

\resizebox{0.98\textwidth}{!}{
\begin{tabular}{lccc}
\toprule
\textbf{Method} & \textbf{Optimization Objective} & \textbf{Role of Curvature} & \textbf{Update form} \\
\midrule

Muon & $\displaystyle \min_{\|D\|_2\le\rho}\langle G,D\rangle $ & None & $\displaystyle D=-\rho\,\operatorname{Polar}(G) $
\\[0.6em]

Newton--Muon
& $\displaystyle \min_Q -\langle G,Q\rangle +\frac{1}{2}\operatorname{tr}(HQAQ^\top) $ 
& Partial ($A$ factor only $\rightarrow$ right preconditioning) & $\displaystyle D\propto-\operatorname{Polar}(GA^{-1})$ \\[0.6em]

Mousse
& $\displaystyle \min_D\langle G,D\rangle \quad \mathrm{s.t.}\quad \|L^{1/4}DR^{1/4}\|_2\le\rho $
& Implicit (Shampoo-style preconditioning)
& $\displaystyle D\propto -L^{-1/4} \operatorname{Polar}(L^{-1/4}GR^{-1/4}) R^{-1/4} $ \\[0.6em]

FISMO
& $\displaystyle \min_D\langle G,D\rangle \quad \mathrm{s.t.}\quad \|P^{1/2}DQ^{1/2}\|_2\le\rho $
& Implicit (Fisher preconditioning)
& $\displaystyle D= -\rho P^{-1/2} \operatorname{Polar}(P^{-1/2}GQ^{-1/2}) Q^{-1/2} $ \\[0.6em]

QSD & 
$\displaystyle \min_{\|D\|_2\le\rho} \left\{ \langle G,D\rangle + \frac{\eta}{2}\operatorname{vec}(D)^\top C\,\operatorname{vec}(D) \right\} $
& Explicit (in the objective) & Frank-Wolfe Iterates
\\

\bottomrule
\end{tabular}
}
\end{table*}

Our contributions are summarized as follows:
\begin{itemize}
\vspace{-0.5em}

\item \textbf{Quadratic spectral optimization.}
We formulate QSD by directly optimizing a quadratic local model under the same spectral-norm constraint as in Muon. We theoretically and empirically show that curvature generally induces a nonuniform spectral allocation and can also rotate the preferred singular subspace. Notably, even under isotropic curvature, QSD remains non-flat while several curvature-aware Muon variants reduce to Muon.

\vspace{-0.2em}

\item \textbf{Scalable curvature-aware solver.}
We develop an inversion-free implementation based on Kronecker-factored curvature and a few online Frank-Wolfe steps. Each inner linear oracle has the same matrix-sign form as Muon, while momentum and curvature smoothing enable effective warm starts.

\vspace{-0.2em}

\item \textbf{Optimization guarantees and empirical gains.}
We derive a computable Frank-Wolfe optimality certificate, show that curvature-aware refinement can improve upon the Muon update under the same quadratic surrogate, and establish an $O(1/K)$ convergence rate for the inner solver. On GPT pre-training, QSD consistently improves validation loss over Muon and recent Muon variants, and reduces wall-clock training time by up to $8.49\%$ at matched validation loss.

\end{itemize}

\vspace{-0.9em}
\subsection{Related Work}
\vspace{-0.5em}
\paragraph{Muon and spectral-norm optimization.} Muon can be interpreted as normalized steepest descent under the spectral norm, or equivalently as linear optimization over a spectral-norm ball \citep{bernstein2024old,crawshaw2025exploration,riabinin2025gluon,pethick2025training}. This viewpoint explains the matrix-sign update as the exact solution of a linear minimization oracle over the spectral-norm ball and connects Muon to more general Linear Minimization Oracle (LMO)-based matrix optimization methods such as Scion \citep{pethick2025training}. Related analyses have studied the geometry, convergence, and implicit properties of Muon-style updates. Our work starts from the same spectral-norm geometry, but asks what update is preferred when the local objective is extended from linear to quadratic.

\vspace{-0.8em}
\paragraph{Curvature-aware and matrix optimization methods.}
Recent work has incorporated curvature or preconditioning into Muon-style optimization. Newton-Muon, FISMO, Mousse, and GO-MUON use activation, Fisher, Shampoo, or K-FAC-style statistics to transform the gradient or define an anisotropic spectral geometry before applying a spectral or polar update \citep{du2026newton,xu2026fismo,zhang2026mousse}. A complementary line directly modifies the singular spectrum of Muon updates through adaptive or nonuniform spectral transformations \citep{dong2026muon,wu2026spectral,wu2026dynmuon,pang2026htmuon}. These methods are related to structured preconditioners such as K-FAC, Shampoo, and SOAP \citep{martens2015optimizing,george2018fast,gupta2018shampoo,morwani2025new,vyas2025soap}, which reduce the cost of full-matrix curvature but still rely on matrix inversions, inverse roots, or eigendecompositions to construct the preconditioner. QSD takes a different route: it keeps the spectral-norm constraint fixed and uses curvature directly in the quadratic objective. It therefore avoids curvature inversion while jointly optimizing the singular directions and singular values of the finite update.

\section{Preliminaries and Problem Formulation}
\label{sec:preliminaries}

We consider the optimization of a matrix-valued parameter $ W \in\mathbb{R}^{m\times n}$ with objective $f(W)$. Let $G=\nabla_W f(W)$ denote the gradient with respect to $W$. We write a single parameter update as $W^{+}=W+\eta D$, where $D\in\mathbb{R}^{m\times n}$ is the update direction and $\eta>0$ is the learning rate. Throughout the paper, $\|\cdot\|_2$ denotes the spectral norm, $\|\cdot\|_*$ the nuclear norm, $\|\cdot\|_F$ the Frobenius norm, and $ \langle X,Y\rangle = \operatorname{tr}(X^\top Y)$ the Frobenius inner product. Building on the constrained interpretation of Muon in Eq.~\ref{eq:intro_muon}, we extend the same spectral-norm geometry from a linear to a quadratic local model as follows:




For a finite update, directions with similar first-order descent can incur substantially different second-order costs. Let $H_W=\nabla^2_{\operatorname{vec}(W)}f(W)$ denote the Hessian and define $\mathcal H[D]=\operatorname{unvec}\left(H_W\operatorname{vec}(D)\right)$. A second-order Taylor expansion gives
\vspace{-0.8em}
\begin{equation}
f(W+\eta D)=f(W)+\eta\langle G,D\rangle+ \frac{\eta^2}{2} \langle D,\mathcal H[D]\rangle + \mathcal R_3(D).
\label{eq:one_step_taylor}
\end{equation}
\vspace{-1.5em}

As the Hessian is expensive to apply, we therefore use a positive-semidefinite curvature operator $\mathcal C$ instead. After removing the constant term and dividing by $\eta$, we obtain
\vspace{-0.5em}
\begin{equation}
q_\eta(D) = \langle G,D\rangle + \frac{\eta}{2} \langle D,\mathcal C[D]\rangle.
\label{eq:curvature_surrogate}
\end{equation}
\vspace{-1.5em}

We define Quadratic Spectral Descent (QSD) as
\begin{equation}
\boxed{
D^\star
\in
\arg\min_{\|D\|_2\leq\rho}
q_\eta(D).
}
\label{eq:curvature_aware_problem}
\end{equation}
When $\mathcal C=0$, QSD reduces exactly to the Muon problem.

\section{Practical Quadratic Spectral Descent}
\label{sec:kfac_fw}

The QSD problem is only useful if its curvature term can be estimated and the constrained quadratic can be solved efficiently. We address these two issues with K-FAC and Frank--Wolfe, respectively.

\vspace{-0.3em}
\subsection{Kronecker-Factored Curvature}
\label{sec:kfac_curvature}
\vspace{-0.3em}

Consider a linear layer $h=Wa,$ where $W\in\mathbb{R}^{m\times n}$ and $a\in\mathbb{R}^{n}$. Let $\delta=\nabla_h\ell\in\mathbb{R}^{m}.$ For a single example, the gradient with respect to $W$ has the outer-product form $\nabla_W\ell=\delta a^\top.$ The corresponding empirical-Fisher block is $C_{\mathrm{EF}}=\mathbb E\left[(aa^\top)\otimes(\delta\delta^\top)\right].$ We use the empirical-Fisher factors as an inexpensive positive-semidefinite structured curvature proxy, without assuming that the empirical Fisher is an exact approximation to the GGN or Hessian \citep{schraudolph2002fast,kunstner2019limitations}. Its overall scale is calibrated against the GGN directional curvature in Section~\ref{sec:practical_qsd}. K-FAC approximates this block by factorizing the expectation \citep{martens2015optimizing,george2018fast}:
\vspace{-0.3em}
\begin{equation}
\widehat C_{\mathrm{KFAC}}= A\otimes B,
\qquad
A=\mathbb E[aa^\top],
\qquad
B=\mathbb E[\delta\delta^\top].
\label{eq:kfac_block}
\vspace{-0.3em}
\end{equation}
Here $A$ captures input-side second-order statistics and $B$ captures output-gradient statistics. The development below only requires $A\succeq0$ and $B\succeq0$, so analogous factorizations can also be used for other positive-semidefinite curvature models. Using $\operatorname{vec}(BDA)=(A^\top\otimes B)\operatorname{vec}(D),$ and the symmetry of $A$, the corresponding matrix-valued curvature action is $\widehat{\mathcal C}[D]=BDA.$

Importantly, K-FAC is used here to define a local quadratic \emph{objective}, rather than merely to construct a preconditioned gradient. We optimize this curvature-aware objective directly over the spectral-norm ball.

\subsection{Frank--Wolfe Solver}
\label{sec:fw_solver}

The QSD problem does not in general have a closed-form solution. We therefore use Frank-Wolfe \citep{frank1956algorithm,jaggi2013revisiting} to solve it approximately. At each step, Frank-Wolfe linearizes the quadratic objective at the current iterate, solves the resulting linear problem over the same spectral-norm ball, and moves toward its solution. Although the original problem is quadratic, each linear subproblem still has a closed-form matrix-sign solution.

Let $D_s$ denote the current inner iterate. The gradient of the quadratic surrogate with respect to $D_s$ is
\begin{equation}
R_s=\nabla_D \widehat q_\eta(D_s)=G+\eta BD_sA.
\label{eq:fw_residual}
\end{equation}
The second term corrects the original gradient $G$ according to the curvature at the current iterate. Frank-Wolfe then solves
\begin{equation}
S_s\in\arg\min_{\|S\|_2\leq\rho}\langle R_s,S\rangle,
\qquad
S_s=-\rho\,\operatorname{msgn}(R_s).
\label{eq:fw_atom}
\end{equation}
The new candidate is obtained by moving toward this feasible atom:
\begin{equation}
D_{s+1}=(1-\gamma_s)D_s+\gamma_sS_s,
\qquad
\gamma_s\in[0,1].
\label{eq:fw_update}
\end{equation}

For the quadratic surrogate, $\gamma_s$ can be obtained by exact line
search. Letting $\Delta_s=S_s-D_s$, we have
\begin{equation}
\gamma_s=\operatorname{clip}_{[0,1]}
\left(
\frac{\langle R_s,D_s-S_s\rangle}{\eta\,\operatorname{tr}\left(\Delta_s^\top B\Delta_s A\right)
}\right),
\label{eq:fw_exact_line_search}
\end{equation}
With $D_0=0$, we have $R_0=G$, and hence the first Frank-Wolfe atom is exactly the Muon direction. Subsequent atoms use the curvature-corrected gradient $R_s$ and progressively refine this initial update.

\subsection{Practical QSD}
\label{sec:practical_qsd}

The derivation above assumes fixed curvature factors. In practice, we estimate the K-FAC factors from minibatch statistics and update them periodically during training.

\paragraph{Factor estimation.} At a factor-refresh iteration $t$, using a sample subset $\mathcal S_t$, we estimate for each layer $\ell$
\begin{equation}
\widehat A_{t,\ell}
=
\frac{1}{|\mathcal S_t|}
\sum_{i\in\mathcal S_t}
a_{i,\ell}a_{i,\ell}^{\top},
\qquad
\widehat B_{t,\ell}
=
\frac{1}{|\mathcal S_t|}
\sum_{i\in\mathcal S_t}
\delta_{i,\ell}\delta_{i,\ell}^{\top},
\label{eq:minibatch_kfac_factors}
\end{equation}
and maintain exponential moving averages
\begin{equation}
A_{t,\ell}
=
\beta_1 A_{t-j,\ell}
+
(1-\beta_1)\widehat A_{t,\ell},
\qquad
B_{t,\ell}
=
\beta_1 B_{t-j,\ell}
+
(1-\beta_1)\widehat B_{t,\ell},
\label{eq:kfac_factor_ema}
\end{equation}
where $j$ is the factor-refresh interval. Between refreshes, the most recent factors are reused. Details could be found in Appendix \ref{app:kfac_factor_estimation}.

\paragraph{Curvature calibration.} To mitigate the possible scale mismatch in the Kronecker approximation, we introduce a per-layer calibration coefficient $\alpha_{t,\ell}>0$ by matching the K-FAC and generalized Gauss--Newton (GGN) directional curvatures along the recent update direction. To reduce noise, the  ratio is smoothed by:
\begin{equation}
\alpha_{t,\ell}
=
\beta_{\alpha}\alpha_{t-k,\ell}
+
(1-\beta_{\alpha})
\operatorname{clip}
\left(
\widehat\alpha_{t,\ell},
\alpha_{\min},
\alpha_{\max}
\right).
\label{eq:calibration_ema}
\end{equation}
Full details are given in Appendix~\ref{app:curvature_calibration}.

Even after calibration, the quadratic surrogate can still differ from the practical finite-step training objective. We therefore apply a mild curvature inflation factor $\tau>1$ to make the update more conservative in directions where curvature already plays a significant role. The effect of this inflation is analyzed in Appendix~\ref{sec:curvature_inflation}. We use $\tau=1.5$ throughout our experiments.

Together with isotropic damping, the practical curvature operator is
\begin{equation}
\widehat{\mathcal C}_{t,\ell}[D]=\tau(\alpha_{t,\ell}B_{t,\ell} D A_{t,\ell}+dD),
\label{eq:practical_curvature_operator}
\end{equation}
where $d\geq0$ is the damping coefficient.

\paragraph{Practical Frank-Wolfe update.} For clarity, we suppress the layer index $\ell$ below. As in Muon, we use the momentum-smoothed matrix $M_t$ in place of the instantaneous gradient. The surrogate gradient at the $s$-th Frank-Wolfe iterate is therefore
\begin{equation}
R_{t,s}=M_t+\eta_t
\left(
\tau (\alpha_t B_tD_{t,s}A_t+dD_{t,s})
\right).
\label{eq:practical_fw_residual}
\end{equation}

Since consecutive inner problems change gradually, we warm-start the solver from the previous solution, $D_{t,0} = D_{t-1,K}$, with $D_{0,0}=0$. The warm start remains feasible and allows a small number of Frank--Wolfe steps to track the evolving QSD solution. Algorithm~\ref{alg:kfac_fw} summarizes the resulting practical update.

\begin{algorithm}[t]
\caption{Practical Quadratic Spectral Descent}
\label{alg:kfac_fw}
\begin{algorithmic}[1]

\REQUIRE Momentum $M_t$, curvature operator $\mathcal C_t$, spectral radius $\rho$, FW steps $K$

\STATE $D_{t,0}\leftarrow D_{t-1,K}$ \COMMENT{use $D_{0,0}=0$}

\FOR{$s=0,\ldots,K-1$}

\STATE $R_{t,s}\leftarrow M_t+\eta_t\,\mathcal C_t[D_{t,s}]$

\STATE $S_{t,s}\leftarrow-\rho\,\operatorname{msgn}(R_{t,s})$

\STATE $\Delta_{t,s}\leftarrow S_{t,s}-D_{t,s}$

\STATE $\gamma_{t,s}\leftarrow
\arg\min_{\gamma\in[0,1]}
q_t(D_{t,s}+\gamma\Delta_{t,s})$

\STATE $D_{t,s+1}\leftarrow D_{t,s}+\gamma_{t,s}\Delta_{t,s}$

\ENDFOR

\STATE $W_{t+1}\leftarrow W_t+\eta_tD_{t,K}$

\end{algorithmic}
\end{algorithm}















\vspace{-0.9em}

\subsection{Optimization guarantees.}
The Frank-Wolfe gap provides a computable certificate of suboptimality for the QSD inner problem. With exact line search, the surrogate objective decreases monotonically; moreover, under zero initialization, the first Frank-Wolfe step is no worse than the Muon update under the same quadratic surrogate. The inner solver converges to the optimal surrogate value at the standard $O(1/K)$ Frank-Wolfe rate. Formal statements and proofs are provided in Appendix~\ref{sec:theory}.

\vspace{-0.9em}
\subsection{Computational Complexity}
\label{sec:complexity}
\vspace{-0.5em}

Despite using curvature information and an iterative inner solver, QSD incurs only a 6\% wall-clock overhead over Muon in our 124M model experiments. This modest overhead is enabled by structured curvature, infrequent factor updates, and warm-started inner solves.

For $W\in\mathbb{R}^{m\times n}$, the full curvature acts on an $mn$-dimensional space, requiring $\mathcal{O}(m^2n^2)$ storage and application cost, and up to $\mathcal{O}(m^3n^3)$ for direct factorization or inversion. QSD instead uses the K-FAC structure $\mathcal{C}_t[D] = B_t D A_t$,  where $A_t\in\mathbb{R}^{n\times n}$ and $B_t\in\mathbb{R}^{m\times m}$. This reduces curvature storage to $\mathcal{O}(m^2+n^2)$ and each curvature application to $\mathcal{O}(m^2n+mn^2)$, without requiring curvature inversion.

The factors are refreshed only periodically and reused between optimizer steps. If they are updated every $j$ steps using $s$ samples, their amortized estimation cost is $\mathcal{O}\left(s(m^2+n^2)/j\right)$. The curvature-scale calibration is performed only every roughly $200$ or more optimizer steps, uses only a small fraction of the training data, and requires only additional forward evaluations without an extra backward pass. Its amortized contribution to the overall training cost is therefore small.

QSD solves the quadratic spectral subproblem using a small number $K$ of Frank-Wolfe steps, each requiring one structured curvature application and one matrix-sign operation. For approximately square matrices, $m\approx n\approx d$, this gives an additional optimizer cost of $\mathcal{O}(Kd^3)$. We further warm start with the previous solution, $D_{t,0}=D_{t-1,K}$, since consecutive subproblems change gradually as the momentum evolves and the curvature factors are updated only periodically. This allows only a few online Frank-Wolfe steps to suffice in practice.

For comparison, the main forward and backward computation of a $d\times d$ layer over $N$ token representations scales as $\mathcal{O}(Nd^2)$, whereas the QSD inner solve scales as $\mathcal{O}(Kd^3)$ and does not grow with the number of processed tokens. Complete wall-clock comparisons are reported in Section~\ref{sec:exp}.

\begin{table*}[t]
\centering
\caption{
Main computational costs for a matrix parameter $W\in\mathbb{R}^{m\times n}$. Here $s$ is the number of samples used to estimate curvature, $j$ is the curvature-factor update interval, and $K$ is the number of Frank--Wolfe steps. }
\vspace{-0.9em}
\label{tab:complexity}

\small
\setlength{\tabcolsep}{4pt}
\renewcommand{\arraystretch}{1.15}

\resizebox{0.8\textwidth}{!}{
\begin{tabular}{lccc}
\toprule
\textbf{Cost} & \textbf{Dense curvature} & \textbf{Muon} & \textbf{QSD} \\
\midrule

Curvature storage & $\mathcal{O}(m^2n^2)$ & -- & $\mathcal{O}(m^2+n^2)$ \\

Curvature estimation & $\mathcal{O}(s m^2n^2)$ & -- & $\displaystyle \mathcal{O}\!\left( s(m^2+n^2)/j \right)$ \\

Curvature application & $\mathcal{O}(m^2n^2)$ & -- & $\mathcal{O}(m^2n+mn^2)$ per FW step \\

Factorization / inversion & $\mathcal{O}(m^3n^3)$ & -- & Not required \\

Matrix-sign operation & -- & $1\times$ & $K\times$ \\

\bottomrule
\end{tabular}
}
\end{table*}
\vspace{-0.9em}

\section{Understanding How Curvature Changes Muon Updates}
\label{sec:beyond_muon}

We now characterize how the quadratic curvature term changes the structure of the optimal update relative to Muon. We isolate two effects. First, when the gradient and curvature share the same singular modes, curvature alone is sufficient to produce nonuniform singular amplitudes. Second, under anisotropic curvature, the optimal update can also change its singular directions, even when the spectral-norm constraint is active. These results show that Muon's flat spectrum and its gradient-aligned singular directions are imposed by the spectral-norm constraint and linear optimization.

\subsection{Curvature Induces Nonuniform Singular Amplitudes}
\label{sec:nonflat_optimum}

We first isolate the effect of curvature on singular amplitudes. Suppose the singular modes of the gradient are also eigenmodes of the curvature operator. Let
\vspace{-0.7em}
\begin{equation}
G= U\operatorname{diag}(g_1,\ldots,g_p)V^\top = \sum_{i=1}^{p}g_i u_i v_i^\top, \qquad g_i\geq0,
\label{eq:gradient_svd_aligned}
\vspace{-1em}
\end{equation}
and define $E_i=u_iv_i^\top$. Assume $\mathcal C[E_i]=\kappa_iE_i$ with $\kappa_i\geq0.$ Under this aligned setting, curvature does not change the singular directions, but it can change the preferred magnitude of each mode.

\begin{proposition}[Optimal spectrum under aligned curvature]
\label{prop:aligned_optimum}
Consider
\begin{equation}
\min_{\|D\|_2\leq\rho}
\left\{
\langle G,D\rangle
+
\frac{\eta}{2}
\langle D,\mathcal C[D]\rangle
\right\},
\label{eq:aligned_problem}
\end{equation}
under aligned curvature. Then there exists an optimal solution of the form
\begin{equation}
D^\star
=
-
U
\operatorname{diag}
\left(
s_1^\star,\ldots,s_p^\star
\right)
V^\top,
\label{eq:aligned_optimum}
\end{equation}
where
\begin{equation}
s_i^\star
=
\begin{cases}
\rho,
&
\kappa_i=0,
\\[2mm]
\displaystyle
\min
\left\{
\rho,
\frac{g_i}{\eta\kappa_i}
\right\},
&
\kappa_i>0.
\end{cases}
\label{eq:optimal_mode_scale}
\end{equation}
\vspace{-1em}
\end{proposition}

The proof is given in
Appendix~\ref{app:proof_aligned_optimum}.

Muon sets every active singular value to $\rho$. In contrast, Proposition~\ref{prop:aligned_optimum} shows that the curvature-aware optimum saturates mode $i$ only when $ g_i\geq\eta\rho\kappa_i$. Otherwise, $s_i^\star=\frac{g_i}{\eta\kappa_i}<\rho$. Thus, even when gradient and curvature share the same singular directions, the quadratic optimum generally has a nonuniform singular-value profile. Intuitively, directions with larger curvature are more costly to move along, so they should receive smaller updates unless the gradient signal is sufficiently strong. 

\subsection{Curvature Can Change the Singular Directions}
\label{sec:singular_direction_adaptation}

We next ask whether curvature can also change the singular directions of the optimal update.  Under K-FAC curvature $\mathcal C[D]=BDA$ with $A\succ0$ and $B\succ0$, the unconstrained quadratic minimizer is $ D_{\mathrm{unc}} = -\frac{1}{\eta}B^{-1}GA^{-1}$. For a rank-one gradient $G=g\,uv^\top$, this becomes $ D_{\mathrm{unc}} = -\frac{g}{\eta} (B^{-1}u)(A^{-1}v)^\top$, which is generally not aligned with the singular directions of $G$. Thus anisotropic curvature naturally induces singular-direction adaptation.

One might ask whether this effect disappears when the spectral-norm constraint is active. The following result shows that it does not.

\begin{proposition}[Direction adaptation with an active spectral constraint]
\label{prop:active_direction_adaptation}
Suppose
\[
\mathcal C[D]=DA,\qquad A\succ0,\qquad
G=g\,uv^\top,
\]
where $g>0$ and $\|u\|_2=\|v\|_2=1$. If the unconstrained minimizer
lies outside the spectral-norm ball, then there exists a unique
$\lambda>0$ such that
\begin{equation}
D^\star
=
-g\,u
\left[(\eta A+\lambda I)^{-1}v\right]^\top,
\qquad
\|D^\star\|_2=\rho.
\end{equation}
If $v$ is not an eigenvector of $A$, the right singular direction of
$D^\star$ is not parallel to that of $G$.
\end{proposition}

The proof is given in Appendix~\ref{app:proof_active_direction_adaptation}. The mechanism is easiest to see in the eigenbasis of the curvature. Let $Aq_j=a_jq_j$ and write $ v=\sum_j c_jq_j$. Then $(\eta A+\lambda I)^{-1}v =\sum_j\frac{c_j}{\eta a_j+\lambda}q_j$. 

Hence, different curvature components of $v$ are rescaled by different amounts: high-curvature components are suppressed more strongly than low-curvature ones. Unless $v$ is a curvature eigenvector, these unequal rescalings change its orientation. The multiplier $\lambda>0$ adjusts the overall shrinkage so that the solution lies on the boundary $\|D^\star\|_2=\rho$, but does not remove this anisotropic reweighting. Thus, the spectral-norm constraint controls the size of the update without forcing it to preserve the singular directions of the gradient.

A simple two-dimensional example makes this explicit. Let
$
B=I,
A=
\begin{pmatrix}
100&0\\
0&1
\end{pmatrix},
G=
\begin{pmatrix}
101&2\\
0&0
\end{pmatrix},
\eta=1,
$
and choose $ \rho=\sqrt{2}$. The unconstrained quadratic minimizer gives $D_{\mathrm{unc}}=-GA^{-1}$. Its spectral norm satisfies $\|D_{\mathrm{unc}}\|_2>\sqrt{2}=\rho$, so the spectral-norm constraint is active. The KKT conditions in Proposition~\ref{prop:active_direction_adaptation} are satisfied with $\lambda=1$, which gives
$D^\star=-G(A+I)^{-1}=-
\begin{pmatrix}
1&1\\
0&0\end{pmatrix},\|D^\star\|_2=\sqrt{2}=\rho$.
The right singular direction of $G$ is proportional to $(101,2)^\top$, whereas that of $D^\star$ is proportional to $(1,1)^\top$. Thus, the optimal singular direction changes even though the solution lies on the boundary of the same spectral-norm ball used by Muon.

\subsection{QSD vs. Muon in Practice}
\label{sec:qsd_vs_muon}

The preceding analysis shows that curvature can change both the singular amplitudes and singular directions of the update. We now examine whether these effects occur in practice. At a fixed training step, we compute the Muon and QSD directions from the same momentum matrix $M_t=U\Sigma V^\top$, thereby excluding differences caused by training trajectories. For this analysis, Muon is computed using the exact polar factor.

We measure spectral and directional deviations by
\begin{equation}
\Delta_{\mathrm{spec}}(D)
=
\frac{\|\sigma(D)-\bar\sigma(D)\mathbf 1\|_2}
{\|\sigma(D)\|_2},
\qquad
\Delta_{\mathrm{dir}}(D;M)
=
\frac{\|D-\mathcal P_M(D)\|_F}{\|D\|_F},
\label{eq:practice_metrics}
\end{equation}
where $\bar\sigma(D)$ is the mean singular value and
$\mathcal P_M(D)=U\operatorname{diag}(\operatorname{diag}(U^\top D V))V^\top$.
Thus, $\Delta_{\mathrm{spec}}$ measures deviation from a flat singular
spectrum, while $\Delta_{\mathrm{dir}}$ measures deviation from the singular
directions of $M$. Both are zero for exact Muon.

Figure~\ref{fig:qsd_vs_muon} shows that QSD differs from Muon in both aspects. Its singular spectrum is clearly nonuniform (Figure~\ref{fig:qsd_vs_muon}(a)), while $U^\top D_{\mathrm{QSD}}V$ contains substantial off-diagonal mass (Figure~\ref{fig:qsd_vs_muon}(b)), indicating that QSD also changes the singular directions. Across layers, both deviations remain consistently nonzero (Figure~\ref{fig:qsd_vs_muon}(c)). These observations confirm that the curvature-induced effects identified above also arise during training.

\begin{figure}[t]
    \centering
    \includegraphics[width=0.8\linewidth]{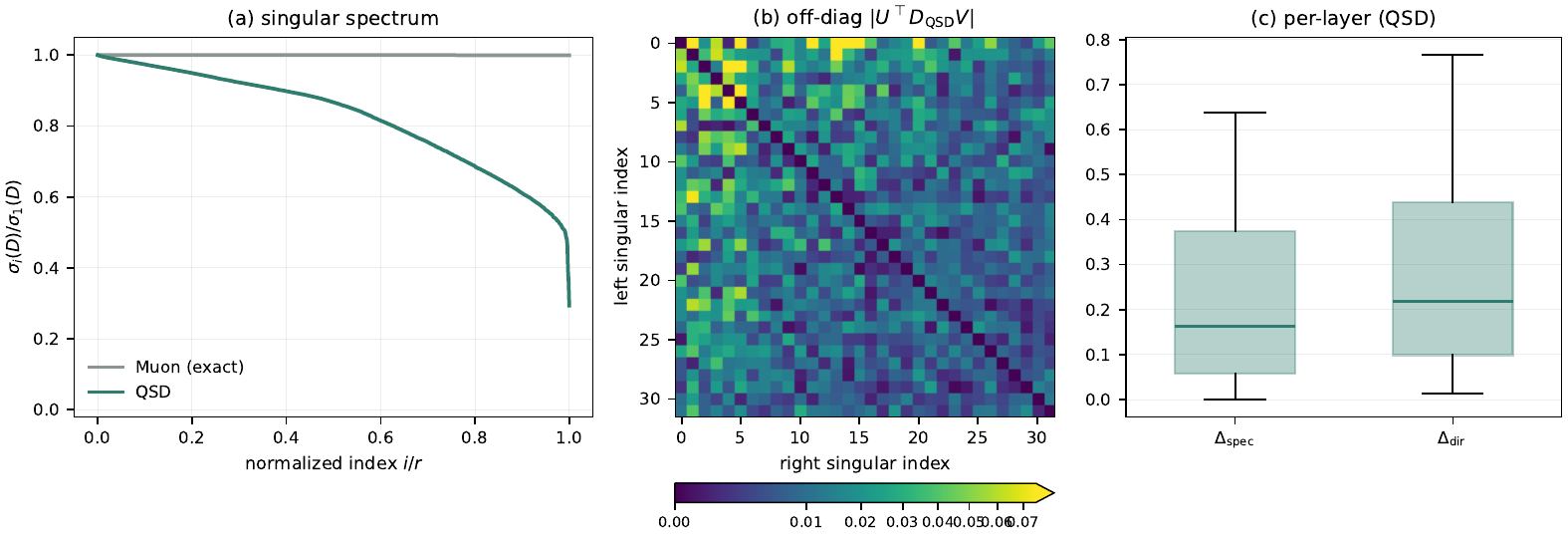}
    \vspace{-0.5em}
    \caption{\textbf{QSD differs from Muon in both singular amplitudes and directions at a mild step of training.} \textbf{(a)} Normalized singular spectra of Muon and QSD for a representative layer. \textbf{(b)} Magnitude of $U^\top D_{\mathrm{QSD}}V$ in the momentum singular direction (top $32\times32$ block). Off-diagonal entries indicate changes in singular directions. \textbf{(c)} Distribution of $\Delta_{\mathrm{spec}}$ and $\Delta_{\mathrm{dir}}$ across QSD layers.
    }
    \label{fig:qsd_vs_muon}
\end{figure}

\vspace{-0.5em}
\section{Experiments}
\label{sec:exp}
\vspace{-0.5em}
\subsection{Pre-training} 
\vspace{-0.5em}
\paragraph{Baseline.} We first pre-train GPT-124M with a compute-optimal \citep{hoffmann2022training} token-to-parameter (T2P) ratio. All models are trained on the FineWeb dataset \citep{penedo2024fineweb}, with the same tokenizer as Muon Speed Run \citep{modded_nanogpt_2024} and Newton-Muon \citep{du2026newton}, with a vocabulary size of 50257 and a batch size of 512. Training is conducted on RTX 5000 Ada GPUs using mixed precision (bfloat16). Following the setup in \citep{modded_nanogpt_2024,du2026newton,zhang2026muon+,amsel2026polar}, we apply QSD (or Muon) to all parameters except embeddings, unembeddings, normalization layers, and positional encodings, which are optimized using AdamW. For the polar operator, we adopt the same configuration as in \citep{jordan2024muon}. We sweep the AdamW learning rate in $[0.001, 0.0018, 0.0032, 0.0056, 0.01]$ and the remaining part in $[0.005, 0.01,0.02, 0.04]$ and report the best results in Table \ref{tab:main_results1}. Detailed hyperparameters are provided in the Appendix \ref{app:experimental_details}. As shown in Table \ref{tab:main_results1}, QSD outperforms Muon at 6\% time overhead. Due to the page limits, we defer the ablation for the baseline to Appendix \ref{sec:abla}.

\begin{table}[h]
\centering
\caption{Performance of pretraining GPT-124M on 4 RTX 5000 Ada GPUs.}
\vspace{-0.9em}
\label{tab:main_results1}
\begin{tabular}{lcccccc}
\toprule
Method & Run 1 & Run 2 & Run 3 & Avg. loss & Avg. time (s) & Time Overhead (\%) \\
\midrule
Muon & 3.3318 & 3.3312 & 3.3319 & 3.3316 & 7390.6 & 0.0 \\
QSD & 3.3151 & 3.3161 & 3.3161  & \textbf{3.3158} (\textcolor{red}{-0.0158}) & 7833.9 & +6.00 \\
\bottomrule
\end{tabular}
\end{table}

\vspace{-0.5em}
\paragraph{GPT-350M.} To verify whether the performance advantage of QSD can be maintained at a larger scale, we pre-trained GPT-350M under the optimal T2P ratio. We reuse the QSD-specific hyperparameters tuned on GPT-124M, and only adjust the learning rate together with the K-FAC factor-refresh and calibration intervals (and their EMA coefficients) to account for the longer training schedule. All other settings are kept unchanged; full values are listed in Table~\ref{tab:training_hyperparameters}. As shown in Table \ref{tab:main_results2}, QSD can still ensure performance gains after the model scale is increased.

\begin{table}[H]
\centering
\caption{Performance of pretraining GPT-350M on 4 RTX 5000 Ada GPUs.}
\vspace{-0.9em}
\label{tab:main_results2}
\begin{tabular}{lccc}
\toprule
Method & Val loss & Time (s) & Time Overhead (\%) \\
\midrule
Muon & 3.0148 & 53569.3 & 0.0 \\
QSD & \textbf{3.0018} (\textcolor{red}{-0.013}) & 56986.3 & +6.38 \\
\bottomrule
\end{tabular}
\end{table}
\vspace{-1.5em}

\subsection{Training Efficiency} 
\vspace{-0.5em}
In this section, we report the token and the wall-clock time required to reach the same validation loss, where the target is set to the final validation loss of the QSD. The results are shown in Table \ref{tab:kfac-eff}. QSD speeds up pretraining by up to 8.49\% in the evaluated settings.

\begin{table}[h]
\centering
\caption{Training cost required to reach a target loss.}
\vspace{-0.9em}
\label{tab:kfac-eff}
\begin{tabular}{lccc}
\toprule
Model & Target loss & Token Speed-up & Wall-clock Time Speed-up  \\
\midrule
GPT-124M & 3.3159 & $\uparrow$ 15\% & $\uparrow$ 8.49\% \\
GPT-350M & 3.0018 & $\uparrow$ 14.3\% & $\uparrow$ 7.4\%  \\
\bottomrule
\end{tabular}
\end{table}
\vspace{-0.5em}

\subsection{Comparison with Muon Variants}
\vspace{-0.5em}
In this section, we compare QSD with the other Muon variants, NorMuon \citep{li2025normuon} and the Newton-Muon \citep{du2026newton}. As shown in Table \ref{tab:main_resultsn}, QSD consistently outperforms both NorMuon and Newton-Muon across all three runs.

\begin{table}[h]
\centering
\caption{Validation performance of GPT-124M on different methods on 4 RTX 5000 Ada GPUs.}
\vspace{-0.9em}
\label{tab:main_resultsn}
\begin{tabular}{lcccc}
\toprule
Method & Run 1 & Run 2 & Run 3 & Avg. loss \\
\midrule
NorMuon & 3.3264 & 3.3280 & 3.3265 & 3.3270\\
Newton Muon & 3.3233 & 3.3237 & 3.3241 & 3.3237 \\
QSD & 3.3151 & 3.3161 & 3.3161  & \textbf{3.3158} \\
\bottomrule
\end{tabular}
\end{table}
\vspace{-0.5em}

\vspace{-0.5em}
\section{Conclusion}
\vspace{-0.5em}
We studied the optimization of quadratic model under the spectral-norm constraint, in order to devise better spectral allocation (and associated spectral subspace update) scheme than Muon for LLM pretraining, thus proposed the QSD algorithm.
Analysis shows that even in the isotropic curvature case, the spectral allocation should not be uniform, rather it need to be  proportional to the \textit{gradient-to-curvature ratio}.  Furthermore, the singular directions of the optimal update should also be modified if the spectral directions of gradient and curvature is not aligned.
QSD incurs only mild complexity increase because it uses inversion-free Kronecker-factored curvature and online Frank--Wolfe updates. The resulting QSD method admits simple optimization guarantees and consistently improves validation loss over Muon and recent Muon variants in GPT pre-training. At matched validation loss, QSD achieves up to an $8.49\%$ wall-clock speedup at matched validation loss in our experiments.

\bibliography{reference}

@article{hoffmann2022training,
  title={Training compute-optimal large language models},
  author={Hoffmann, Jordan and Borgeaud, Sebastian and Mensch, Arthur and Buchatskaya, Elena and Cai, Trevor and Rutherford, Eliza and Casas, Diego de Las and Hendricks, Lisa Anne and Welbl, Johannes and Clark, Aidan and others},
  journal={arXiv preprint arXiv:2203.15556},
  year={2022}
}

@article{loshchilov2017decoupled,
  title={Decoupled weight decay regularization},
  author={Loshchilov, Ilya and Hutter, Frank},
  journal={arXiv preprint arXiv:1711.05101},
  year={2017}
}

@article{jordan2024muon,
  title={Muon: An optimizer for hidden layers in neural networks, 2024},
  author={Jordan, Keller and Jin, Yuchen and Boza, Vlado and You, Jiacheng and Cesista, Franz and Newhouse, Laker and Bernstein, Jeremy},
  journal={URL https://kellerjordan. github. io/posts/muon},
  volume={6},
  number={3},
  pages={4},
  year={2024}
}

@article{liu2025muon,
  title={Muon is scalable for llm training},
  author={Liu, Jingyuan and Su, Jianlin and Yao, Xingcheng and Jiang, Zhejun and Lai, Guokun and Du, Yulun and Qin, Yidao and Xu, Weixin and Lu, Enzhe and Yan, Junjie and others},
  journal={arXiv preprint arXiv:2502.16982},
  year={2025}
}

@article{shah2025practical,
  title={Practical efficiency of muon for pretraining},
  author={Shah, Ishaan and Polloreno, Anthony M and Stratos, Karl and Monk, Philip and Chaluvaraju, Adarsh and Hojel, Andrew and Ma, Andrew and Thomas, Anil and Tanwer, Ashish and Shah, Darsh J and others},
  journal={arXiv preprint arXiv:2505.02222},
  year={2025}
}

@inproceedings{wen2026fantastic,
  title={Fantastic pretraining optimizers and where to find them},
  author={Wen, Kaiyue and Hall, David and Ma, Tengyu and Liang, Percy},
  booktitle={International Conference on Learning Representations},
  volume={2026},
  pages={144731--144838},
  year={2026}
}

@article{pethick2025training,
  title={Training deep learning models with norm-constrained lmos},
  author={Pethick, Thomas and Xie, Wanyun and Antonakopoulos, Kimon and Zhu, Zhenyu and Silveti-Falls, Antonio and Cevher, Volkan},
  journal={arXiv preprint arXiv:2502.07529},
  year={2025}
}

@article{du2026newton,
  title={The newton-muon optimizer},
  author={Du, Zhehang and Su, Weijie},
  journal={arXiv preprint arXiv:2604.01472},
  year={2026}
}

@article{xu2026fismo,
  title={FISMO: Fisher-structured momentum-orthogonalized optimizer},
  author={Xu, Chenrui and Yan, Wenjing and Zhang, Ying-Jun Angela},
  journal={arXiv preprint arXiv:2601.21750},
  year={2026}
}

@article{zhang2026mousse,
  title={Mousse: Rectifying the geometry of muon with curvature-aware preconditioning},
  author={Zhang, Yechen and Xing, Shuhao and Huang, Junhao and Lv, Kai and Zhou, Yunhua and Qiu, Xipeng and Guo, Qipeng and Chen, Kai},
  journal={arXiv preprint arXiv:2603.09697},
  year={2026}
}

@article{wu2026spectral,
  title={Spectral Allocation: Why Muon Outperforms Adam, and How to Improve Muon},
  author={Wu, Xiaodong and Yu, Wenyi and Zhang, Chao and Woodland, Philip},
  journal={arXiv preprint arXiv:2608.25990},
  year={2026}
}

@article{dong2026muon,
  title = {{Muon}$^{p}$: Muon with Fractional Spectral Powers},
  author={Dong, Yihe and Sawin, Will},
  journal={arXiv preprint arXiv:2606.13867},
  year={2026}
}

@article{bernstein2024old,
  title={Old optimizer, new norm: An anthology},
  author={Bernstein, Jeremy and Newhouse, Laker},
  journal={arXiv preprint arXiv:2409.20325},
  year={2024}
}

@article{crawshaw2025exploration,
  title={An exploration of non-euclidean gradient descent: Muon and its many variants},
  author={Crawshaw, Michael and Modi, Chirag and Liu, Mingrui and Gower, Robert M},
  journal={arXiv preprint arXiv:2510.09827},
  year={2025}
}

@article{riabinin2025gluon,
  title={Gluon: Making muon \& scion great again!(bridging theory and practice of lmo-based optimizers for llms)},
  author={Riabinin, Artem and Shulgin, Egor and Gruntkowska, Kaja and Richt{\'a}rik, Peter},
  journal={arXiv preprint arXiv:2505.13416},
  year={2025}
}

@inproceedings{martens2015optimizing,
  title={Optimizing neural networks with kronecker-factored approximate curvature},
  author={Martens, James and Grosse, Roger},
  booktitle={International conference on machine learning},
  pages={2408--2417},
  year={2015},
  organization={PMLR}
}

@article{george2018fast,
  title={Fast approximate natural gradient descent in a kronecker factored eigenbasis},
  author={George, Thomas and Laurent, C{\'e}sar and Bouthillier, Xavier and Ballas, Nicolas and Vincent, Pascal},
  journal={Advances in neural information processing systems},
  volume={31},
  year={2018}
}

@inproceedings{gupta2018shampoo,
  title={Shampoo: Preconditioned stochastic tensor optimization},
  author={Gupta, Vineet and Koren, Tomer and Singer, Yoram},
  booktitle={International Conference on Machine Learning},
  pages={1842--1850},
  year={2018},
  organization={PMLR}
}

@inproceedings{morwani2025new,
  title={A new perspective on shampoo's preconditioner},
  author={Morwani, Depen and Shapira, Itai and Vyas, Nikhil and Malach, Eran and Kakade, Sham and Janson, Lucas},
  booktitle={International Conference on Learning Representations},
  volume={2025},
  pages={5802--5822},
  year={2025}
}

@inproceedings{vyas2025soap,
  title={SOAP: Improving and stabilizing shampoo using adam for language modeling},
  author={Vyas, Nikhil and Morwani, Depen and Zhao, Rosie and Shapira, Itai and Brandfonbrener, David and Janson, Lucas and Kakade, Sham},
  booktitle={International Conference on Learning Representations},
  volume={2025},
  pages={93423--93444},
  year={2025}
}

@inproceedings{jaggi2013revisiting,
  title={Revisiting Frank-Wolfe: Projection-free sparse convex optimization},
  author={Jaggi, Martin},
  booktitle={International conference on machine learning},
  pages={427--435},
  year={2013},
  organization={PMLR}
}

@article{schraudolph2002fast,
  title={Fast curvature matrix-vector products for second-order gradient descent},
  author={Schraudolph, Nicol N},
  journal={Neural computation},
  volume={14},
  number={7},
  pages={1723--1738},
  year={2002},
  publisher={MIT Press}
}

@article{wu2026dynmuon,
  title={DynMuon: A Dynamic Spectral Shaping View of Muon},
  author={Wu, Fangzhou and Shah, Rikhav and Silwal, Sandeep and Zhang, Qiuyi},
  journal={arXiv preprint arXiv:2605.17109},
  year={2026}
}

@inproceedings{pang2026htmuon,
  title={Htmuon: Improving muon via heavy-tailed spectral correction},
  author={Pang, Tianyu and Fang, Yujie and Liu, Zihang and Deng, Shenyang and Hsiung, Lei and Yu, Shuhua and Yang, Yaoqing},
  booktitle={Findings of the Association for Computational Linguistics: ACL 2026},
  pages={36504--36535},
  year={2026}
}

@article{wang2026muon,
  title={Why Muon Outperforms Adam: A Curvature Perspective},
  author={Wang, Shuche and Zhang, Fengzhuo and Li, Jiaxiang and Bergemann, Dirk and Yang, Zhuoran},
  journal={arXiv preprint arXiv:2606.04662},
  year={2026}
}

@article{frank1956algorithm,
  title={An algorithm for quadratic programming},
  author={Frank, Marguerite and Wolfe, Philip},
  journal={Naval research logistics quarterly},
  volume={3},
  number={1-2},
  pages={95--110},
  year={1956},
  publisher={Wiley Online Library}
}

@article{penedo2024fineweb,
  title={The fineweb datasets: Decanting the web for the finest text data at scale},
  author={Penedo, Guilherme and Kydl{\'\i}{\v{c}}ek, Hynek and Lozhkov, Anton and Mitchell, Margaret and Raffel, Colin and Von Werra, Leandro and Wolf, Thomas and others},
  journal={Advances in Neural Information Processing Systems},
  volume={37},
  pages={30811--30849},
  year={2024}
}

@misc{modded_nanogpt_2024,
  author       = {Keller Jordan and Jeremy Bernstein and Brendan Rappazzo and
                  @fernbear.bsky.social and Boza Vlado and You Jiacheng and
                  Franz Cesista and Braden Koszarsky and @Grad62304977},
  title        = {modded-nanogpt: Speedrunning the NanoGPT baseline},
  year         = {2024},
  url          = {https://github.com/KellerJordan/modded-nanogpt}
}

@article{zhang2026muon+,
  title={Muon+: Towards better muon via one additional normalization step},
  author={Zhang, Ruijie and Zhao, Yequan and Liu, Ziyue and Wang, Zhengyang and Zhang, Zheng},
  journal={arXiv e-prints},
  pages={arXiv--2602},
  year={2026}
}

@inproceedings{amsel2026polar,
  title={The polar express: Optimal matrix sign methods and their application to the muon algorithm},
  author={Amsel, Noah and Persson, David and Musco, Christopher and Gower, Robert M},
  booktitle={International Conference on Learning Representations},
  volume={2026},
  pages={138323--138360},
  year={2026}
}

@article{li2025normuon,
  title={Normuon: Making muon more efficient and scalable},
  author={Li, Zichong and Liu, Liming and Liang, Chen and Chen, Weizhu and Zhao, Tuo},
  journal={arXiv preprint arXiv:2510.05491},
  year={2025}
}

@article{su2025isotropic,
  title={Isotropic curvature model for understanding deep learning optimization: Is gradient orthogonalization optimal?},
  author={Su, Weijie},
  journal={arXiv preprint arXiv:2511.00674},
  year={2025}
}

@article{kunstner2019limitations,
  title={Limitations of the empirical fisher approximation for natural gradient descent},
  author={Kunstner, Frederik and Hennig, Philipp and Balles, Lukas},
  journal={Advances in neural information processing systems},
  volume={32},
  year={2019}
}
\bibliographystyle{plain}

\appendix

\section{Organization of the Appendix}
\label{app:organization}

The appendix is organized as follows:
\begin{itemize}
    \item Sec.~\ref{sec:theory}: theoretical guarantees for the QSD inner solver, including the Frank--Wolfe optimality certificate and the $O(1/K)$ convergence rate.
    \item Sec.~\ref{app:theory_proofs}: detailed proofs of the theoretical results on curvature-aware spectral allocation, singular-direction adaptation, the Frank--Wolfe optimality certificate, and inner-loop convergence.
    \item Sec.~\ref{app:proof_fw_iterate_convergence}: convergence of the QSD iterates under positive damping.
    \item Sec.~\ref{app:proof_fw_monotonic}: monotonic surrogate descent under exact line search and the relation between QSD and the Muon update.
    \item Sec.~\ref{sec:curvature_inflation}: analysis of the curvature inflation factor and its effect on the spectral allocation.
    \item Sec.~\ref{sec:abla}: mechanism ablations and sensitivity analyses for the QSD-specific hyperparameters.
    \item Sec.~\ref{app:kfac_factor_estimation}: implementation details of K-FAC factor estimation.
    \item Sec.~\ref{app:curvature_calibration}: details of the GGN-based directional curvature calibration and its finite-difference implementation.
    \item Sec.~\ref{app:experimental_details}: detailed training settings and hyperparameters used in the GPT pre-training experiments.
    \item Sec.~\ref{sec:limitations}: limitations of QSD, its theoretical guarantees, and the current experimental evaluation.
\end{itemize}

\clearpage
\section{Theoretical Guarantees}
\label{sec:theory}
\vspace{-0.5em}

We analyze the inner Frank--Wolfe solver used by QSD. The Frank--Wolfe gap provides a computable certificate of surrogate suboptimality, while the iterates converge at an $O(1/K)$ rate. We also relate the resulting update to Muon under the same quadratic surrogate. All proofs are deferred to Appendix~\ref{app:theory_proofs}.

\subsection{Optimality Certificate}
\label{sec:fw_optimality_certificate}
\vspace{-0.5em}

Consider the convex quadratic surrogate
\begin{equation}
\widehat q_\eta(D)
=
\langle G,D\rangle
+
\frac{\eta}{2}
\langle D,\widehat{\mathcal C}[D]\rangle,
\qquad
\|D\|_2\leq\rho,
\label{eq:theory_surrogate}
\end{equation}
where $\widehat{\mathcal C}$ is positive semidefinite and self-adjoint. Let $ \widehat D^\star \in \arg\min_{\|D\|_2\leq\rho} \widehat q_\eta(D)$ denote an optimal solution. At Frank--Wolfe iteration $s$, define
\begin{equation}
R_s
=
\nabla_D\widehat q_\eta(D_s)
=
G+\eta\widehat{\mathcal C}[D_s].
\label{eq:theory_fw_residual}
\end{equation}
The linear minimization oracle is
\begin{equation}
S_s
\in
\arg\min_{\|S\|_2\leq\rho}
\langle R_s,S\rangle
=
-\rho\,\operatorname{msgn}(R_s),
\label{eq:theory_fw_atom}
\end{equation}
and the corresponding Frank--Wolfe gap is
\begin{equation}
\mathcal G_s
=
\langle R_s,D_s-S_s\rangle
=
\langle R_s,D_s\rangle
+
\rho\|R_s\|_*,
\label{eq:fw_gap_closed_form}
\end{equation}
where the second equality follows from spectral--nuclear norm duality.

\begin{proposition}[Frank--Wolfe optimality certificate]
\label{prop:fw_gap_certificate}
For every feasible $D_s$,
\begin{equation}
0
\leq
\widehat q_\eta(D_s)
-
\widehat q_\eta(\widehat D^\star)
\leq
\mathcal G_s.
\label{eq:fw_gap_certificate}
\end{equation}
In particular, $\mathcal G_s=0$ if and only if $D_s$ is optimal.
\end{proposition}

The proof is given in Appendix~\ref{app:proof_fw_gap}. Intuitively, $\mathcal G_s$ measures how much first-order improvement is still available within the spectral-norm ball. A small gap therefore indicates that no feasible direction can substantially improve the current surrogate value.

Exact line search monotonically decreases the quadratic surrogate, with strict decrease whenever the Frank--Wolfe gap is nonzero. Moreover, with zero initialization, the first Frank--Wolfe atom is exactly the Muon direction. Since the line search optimizes over the segment between zero and this atom, $\widehat q_\eta(D_1) \leq \widehat q_\eta(D_{\mathrm M})$, and subsequent Frank--Wolfe steps preserve this inequality by monotonic descent. Formal statements and proofs are provided in Appendix~\ref{app:proof_fw_monotonic} and Appendix~\ref{sec:theory_relation_muon}.

\subsection{Convergence of the QSD Inner Solver}
\label{sec:fw_convergence}
\vspace{-0.5em}

We now specialize to the practical QSD surrogate at a fixed training iteration $t$:
\begin{equation}
\widehat q_t(D)
=
\langle M_t,D\rangle
+
\frac{\eta_t\tau}{2}
\left[
\alpha_t\operatorname{tr}
\left(
D^\top B_tDA_t
\right)
+
d\|D\|_F^2
\right],
\qquad
\|D\|_2\leq\rho.
\label{eq:practical_quadratic_surrogate}
\end{equation}
Assume $ A_t\succeq0, B_t\succeq0, \alpha_t>0, d\geq0$,  and let $ D_t^\star \in \arg\min_{\|D\|_2\leq\rho} \widehat q_t(D) $ denote an optimal solution.

\begin{theorem}[Inner-loop convergence of QSD]
\label{thm:kfac_fw_inner_convergence}
Let $D_{t,K}$ denote the iterate after $K\geq1$ Frank--Wolfe refinements with exact line search, starting from any feasible initialization $D_{t,0}$. Then, with $r=\min\{m,n\}$,
\begin{equation}
\widehat q_t(D_{t,K})
-
\widehat q_t(D_t^\star)
\leq
\frac{
8\eta_t\tau r\rho^2
\left(
\alpha_t\|A_t\|_2\|B_t\|_2+d
\right)
}{
K+2
}.
\label{eq:kfac_fw_inner_convergence}
\end{equation}
Hence, the inner solver converges to the optimal surrogate value at rate $O(K^{-1})$.
\end{theorem}

The proof is given in Appendix~\ref{app:proof_fw_convergence}. This matches the classical $O(1/K)$ convergence rate of Frank--Wolfe for smooth convex objectives \citep{jaggi2013revisiting}, with the constant here specialized to the spectral-norm feasible set and the K-FAC quadratic curvature. 

Because Theorem~\ref{thm:kfac_fw_inner_convergence} holds for any feasible initialization, it also applies directly to the practical warm start $D_{t,0}=D_{t-1,K}$. When $d>0$, the surrogate is additionally strongly convex, which implies convergence of the iterate itself to the unique inner optimum. The corresponding bound is given in Appendix~\ref{app:proof_fw_iterate_convergence}.

\clearpage
\section{Proofs of Theoretical Results}
\label{app:theory_proofs}

This appendix provides proofs for the results stated in
Sections~\ref{sec:beyond_muon} and Appendix~\ref{sec:theory}.

\subsection{Proof of Proposition~\ref{prop:aligned_optimum}}
\label{app:proof_aligned_optimum}

Let
$$
\mathcal S
=
\operatorname{span}\{E_1,\ldots,E_p\}.
$$
Because every $E_i$ is an eigenmode of the self-adjoint operator $\mathcal C$, the subspace $\mathcal S$ is invariant under $\mathcal C$. Its orthogonal complement $\mathcal S^\perp$ is also invariant.

For any feasible $D$, write
\begin{equation}
D
=
D_{\parallel}+D_{\perp},
\qquad
D_{\parallel}\in\mathcal S,
\qquad
D_{\perp}\in\mathcal S^\perp.
\label{eq:app_aligned_decomposition}
\end{equation}
Since $G\in\mathcal S$,
$$
\langle G,D_{\perp}\rangle=0.
$$
Moreover, self-adjointness and invariance imply
$$
\left\langle
D_{\parallel},
\mathcal C[D_{\perp}]
\right\rangle
=
0.
$$
Therefore,
\begin{align}
q_\eta(D)
&=
\langle G,D_{\parallel}\rangle
+
\frac{\eta}{2}
\left\langle
D_{\parallel},
\mathcal C[D_{\parallel}]
\right\rangle
\nonumber\\
&\qquad
+
\frac{\eta}{2}
\left\langle
D_{\perp},
\mathcal C[D_{\perp}]
\right\rangle
\nonumber\\
&\geq
q_\eta(D_{\parallel}),
\label{eq:app_discard_orthogonal_component}
\end{align}
where the inequality follows from $\mathcal C\succeq0$.

We next verify that $D_{\parallel}$ remains feasible. Write
$$
D_{\parallel}
=
\sum_{i=1}^{p}d_iE_i,
\qquad
d_i
=
\langle E_i,D\rangle
=
u_i^\top Dv_i.
$$
For every $i$,
$$
|d_i|
=
|u_i^\top Dv_i|
\leq
\|D\|_2.
$$
Since the matrices $E_i=u_iv_i^\top$ share orthonormal left and right singular vectors,
$$
\|D_{\parallel}\|_2
=
\max_i|d_i|
\leq
\|D\|_2
\leq
\rho.
$$
Thus $D_{\parallel}$ is feasible, and there exists an optimal solution entirely in $\mathcal S$.

We may therefore write
\begin{equation}
D
=
-
\sum_{i=1}^{p}s_iE_i
=
-
U\operatorname{diag}(s_1,\ldots,s_p)V^\top.
\label{eq:app_aligned_candidate}
\end{equation}
Its spectral norm is
$$
\|D\|_2
=
\max_i|s_i|,
$$
so feasibility is equivalent to
$$
|s_i|\leq\rho
\qquad
\text{for all }i.
$$

Using
$$
\mathcal C[E_i]=\kappa_iE_i
$$
and the orthonormality of the $E_i$'s, the objective becomes
\begin{equation}
q_\eta(D)
=
\sum_{i=1}^{p}
\left(
-g_is_i
+
\frac{\eta\kappa_i}{2}s_i^2
\right).
\label{eq:app_aligned_scalar_problem}
\end{equation}
Because $g_i\geq0$, an optimizer may be chosen with $s_i\geq0$. The problem therefore decouples into
\begin{equation}
\min_{0\leq s_i\leq\rho}
\left\{
-g_is_i
+
\frac{\eta\kappa_i}{2}s_i^2
\right\}.
\label{eq:app_modewise_problem}
\end{equation}

If $\kappa_i>0$, the unconstrained minimizer satisfies
$$
-g_i+\eta\kappa_i s_i=0,
$$
and therefore
$$
s_i^{\mathrm{unc}}
=
\frac{g_i}{\eta\kappa_i}.
$$
Clipping to the feasible interval gives
$$
s_i^\star
=
\min
\left\{
\rho,
\frac{g_i}{\eta\kappa_i}
\right\}.
$$

If $\kappa_i=0$, the mode-wise objective becomes $-g_is_i$, which is minimized at $s_i^\star=\rho$ whenever $g_i>0$. If $g_i=0$, every $s_i\in[0,\rho]$ is optimal, so $s_i^\star=\rho$ may be selected without loss of generality. This proves Proposition~\ref{prop:aligned_optimum}.




\subsection{Proof of Proposition~\ref{prop:active_direction_adaptation}}
\label{app:proof_active_direction_adaptation}

\begin{proof}
Consider
\begin{equation}
\min_{\|D\|_2\leq\rho}
\left\{
\langle G,D\rangle
+
\frac{\eta}{2}
\operatorname{tr}(D^\top D A)
\right\},
\label{eq:app_active_direction_problem}
\end{equation}
where
\[
G=g\,uv^\top,
\qquad
A\succ0,
\qquad
\|u\|_2=\|v\|_2=1.
\]

For any feasible $D$, decompose
\begin{equation}
D
=
u x^\top + D_\perp,
\qquad
u^\top D_\perp=0.
\label{eq:app_active_direction_decomposition}
\end{equation}
Since $G=g\,uv^\top$,
\begin{equation}
\langle G,D_\perp\rangle=0.
\end{equation}
Moreover,
\begin{align}
\operatorname{tr}(D^\top D A)
&=
x^\top A x
+
\operatorname{tr}(D_\perp^\top D_\perp A)
\nonumber\\
&\geq
x^\top A x,
\end{align}
because $A\succ0$. Also,
\begin{equation}
\|u x^\top\|_2
=
\|x\|_2
\leq
\|D\|_2.
\end{equation}
Therefore, there exists an optimal solution of the form
\begin{equation}
D^\star=u{x^\star}^\top,
\end{equation}
where $x^\star$ solves
\begin{equation}
\min_{\|x\|_2\leq\rho}
\left\{
g\,v^\top x
+
\frac{\eta}{2}x^\top A x
\right\}.
\label{eq:app_active_direction_vector_problem}
\end{equation}

The unconstrained minimizer is
\begin{equation}
x_{\mathrm{unc}}
=
-\frac{g}{\eta}A^{-1}v.
\end{equation}
By assumption,
\begin{equation}
\|x_{\mathrm{unc}}\|_2>\rho,
\end{equation}
so the Euclidean-norm constraint in Eq.~\ref{eq:app_active_direction_vector_problem} is active.

The KKT conditions are therefore
\begin{equation}
g v+\eta A x^\star+\lambda x^\star=0,
\qquad
\lambda>0,
\qquad
\|x^\star\|_2=\rho.
\label{eq:app_active_direction_kkt}
\end{equation}
Solving the stationarity condition gives
\begin{equation}
x^\star
=
-g(\eta A+\lambda I)^{-1}v.
\label{eq:app_active_direction_xstar}
\end{equation}
Hence,
\begin{equation}
D^\star
=
-g\,u
\left[
(\eta A+\lambda I)^{-1}v
\right]^\top,
\end{equation}
with $\lambda>0$ chosen so that
$\|D^\star\|_2=\|x^\star\|_2=\rho$.

It remains to compare the right singular directions. Since $G$ is rank one, its right singular direction is $v$. The right singular direction of $D^\star$ is proportional to $(\eta A+\lambda I)^{-1}v$.

If these two directions were parallel, then for some scalar $c\neq0$,
\begin{equation}
(\eta A+\lambda I)^{-1}v
=
c v.
\end{equation}
Multiplying by $\eta A+\lambda I$ gives
\begin{equation}
v
=
c(\eta A+\lambda I)v,
\end{equation}
and therefore
\begin{equation}
Av
=
\frac{1/c-\lambda}{\eta}v.
\end{equation}
Thus, $v$ would be an eigenvector of $A$.

Consequently, if $v$ is not an eigenvector of $A$,
\begin{equation}
(\eta A+\lambda I)^{-1}v
\not\parallel v,
\end{equation}
and the right singular direction of the constrained optimum differs from that of the gradient. Since $\lambda>0$ and $\|D^\star\|_2=\rho$, this direction change occurs while the spectral-norm constraint is active.
\end{proof}

\subsection{Proof of Proposition~\ref{prop:fw_gap_certificate}}
\label{app:proof_fw_gap}

\begin{proof}
Since $\widehat q_\eta$ is convex, for any feasible $D$,
\begin{equation}
\widehat q_\eta(D)
\geq
\widehat q_\eta(D_s)
+
\left\langle
\nabla_D\widehat q_\eta(D_s),
D-D_s
\right\rangle.
\end{equation}
Taking $D=\widehat D^\star$ and using $R_s=\nabla_D\widehat q_\eta(D_s)$ gives
\begin{equation}
\widehat q_\eta(D_s)
-
\widehat q_\eta(\widehat D^\star)
\leq
\langle R_s,D_s-\widehat D^\star\rangle.
\end{equation}

By definition, $S_s$ minimizes the linearized objective over the same feasible set, and therefore
\begin{equation}
\langle R_s,S_s\rangle
\leq
\langle R_s,\widehat D^\star\rangle.
\end{equation}
Hence,
\begin{align}
\widehat q_\eta(D_s)
-
\widehat q_\eta(\widehat D^\star)
&\leq
\langle R_s,D_s-\widehat D^\star\rangle
\\
&\leq
\langle R_s,D_s-S_s\rangle
\\
&=
\mathcal G_s.
\end{align}
The lower bound follows from the optimality of $\widehat D^\star$.

It remains to characterize the zero-gap case. Since $\mathcal G_s\geq0$, if $\mathcal G_s=0$, the bound above implies
\begin{equation}
\widehat q_\eta(D_s)
=
\widehat q_\eta(\widehat D^\star),
\end{equation}
so $D_s$ is optimal. Conversely, if $D_s$ is optimal, the first-order optimality condition for convex optimization over the feasible set is
\begin{equation}
\langle R_s,D-D_s\rangle
\geq0
\qquad
\text{for all feasible }D.
\end{equation}
In particular, taking $D=S_s$ gives
\begin{equation}
\langle R_s,D_s-S_s\rangle
\leq0.
\end{equation}
Since the Frank--Wolfe gap is always nonnegative, this implies
$\mathcal G_s=0$.
\end{proof}

\subsection{Proof of Theorem~\ref{thm:kfac_fw_inner_convergence}}
\label{app:proof_fw_convergence}

\begin{proof}
Fix a training iteration $t$. The practical QSD surrogate is
\begin{equation}
\widehat q_t(D)
=
\langle M_t,D\rangle
+
\frac{\eta_t\tau}{2}
\left[
\alpha_t
\operatorname{tr}
\left(
D^\top B_tDA_t
\right)
+
d\|D\|_F^2
\right].
\label{eq:app_practical_surrogate}
\end{equation}
Its gradient is
\begin{equation}
\nabla_D\widehat q_t(D)
=
M_t
+
\eta_t\tau
\left(
\alpha_t B_tDA_t+dD
\right).
\label{eq:app_practical_gradient}
\end{equation}

Define
\begin{equation}
L_t
=
\eta_t\tau
\left(
\alpha_t\|A_t\|_2\|B_t\|_2+d
\right).
\label{eq:proof_Lt}
\end{equation}
For any matrices $D$ and $D'$,
\begin{align}
&
\left\|
\nabla_D\widehat q_t(D)
-
\nabla_D\widehat q_t(D')
\right\|_F
\nonumber\\
&\quad=
\eta_t\tau
\left\|
\alpha_t B_t(D-D')A_t
+
d(D-D')
\right\|_F
\nonumber\\
&\quad\leq
\eta_t\tau
\left(
\alpha_t\|B_t\|_2\|A_t\|_2+d
\right)
\|D-D'\|_F
\nonumber\\
&\quad=
L_t\|D-D'\|_F.
\label{eq:proof_gradient_lipschitz}
\end{align}
Thus, $\widehat q_t$ has an $L_t$-Lipschitz gradient with respect to the Frobenius norm.

Let
\begin{equation}
\mathcal D
=
\{D:\|D\|_2\leq\rho\}.
\end{equation}
For any $D\in\mathcal D$, with $r=\min\{m,n\}$,
\begin{equation}
\|D\|_F
\leq
\sqrt r\,\|D\|_2
\leq
\sqrt r\,\rho.
\end{equation}
Hence, for any $D,D'\in\mathcal D$,
\begin{equation}
\|D-D'\|_F
\leq
2\sqrt r\,\rho.
\label{eq:proof_domain_diameter}
\end{equation}

The Frank--Wolfe curvature constant of $\widehat q_t$ over $\mathcal D$ therefore satisfies
\begin{equation}
C_t
\leq
L_t
\operatorname{diam}_F(\mathcal D)^2.
\end{equation}
Using Eq.~\ref{eq:proof_domain_diameter},
\begin{align}
C_t
&\leq
4L_t r\rho^2
\\
&=
4\eta_t\tau r\rho^2
\left(
\alpha_t\|A_t\|_2\|B_t\|_2+d
\right).
\label{eq:proof_curvature_constant}
\end{align}

Let
\begin{equation}
h_s
=
\widehat q_t(D_{t,s})
-
\widehat q_t(D_t^\star).
\end{equation}
The standard Frank--Wolfe descent inequality gives, for any
$\gamma\in[0,1]$,
\begin{equation}
h_{s+1}
\leq
(1-\gamma)h_s
+
\frac{\gamma^2}{2}C_t.
\label{eq:proof_fw_recurrence}
\end{equation}
Exact line search performs at least as well as any prescribed
$\gamma$.

At the first iteration, choosing $\gamma=1$ gives
\begin{equation}
h_1
\leq
\frac{C_t}{2}.
\end{equation}
For subsequent iterations, the standard Frank--Wolfe \citep{jaggi2013revisiting} argument with $\gamma_s=2/(s+2)$ yields
\begin{equation}
h_K
\leq
\frac{2C_t}{K+2},
\qquad
K\geq1.
\end{equation}
Substituting Eq.~\ref{eq:proof_curvature_constant} gives
\begin{equation}
\boxed{
\widehat q_t(D_{t,K})
-
\widehat q_t(D_t^\star)
\leq
\frac{
8\eta_t\tau r\rho^2
\left(
\alpha_t\|A_t\|_2\|B_t\|_2+d
\right)
}{
K+2
}.
}
\end{equation}
This proves the result.
\end{proof}

\clearpage
\section{Iterate Convergence under Damping}
\label{app:proof_fw_iterate_convergence}

For any perturbation $H$, the quadratic form induced by the Hessian of $\widehat q_t$ is
\begin{align}
\left\langle
H,
\nabla^2\widehat q_t[H]
\right\rangle
&=
\eta_t\tau
\left[
\alpha_t
\operatorname{tr}
\left(
H^\top B_tHA_t
\right)
+
d\|H\|_F^2
\right].
\end{align}
Since $A_t\succeq0$, $B_t\succeq0$, and $\alpha_t>0$,
\begin{equation}
\alpha_t
\operatorname{tr}
\left(
H^\top B_tHA_t
\right)
\geq0.
\end{equation}
Therefore, when $d>0$,
\begin{equation}
\left\langle
H,
\nabla^2\widehat q_t[H]
\right\rangle
\geq
\eta_t\tau d\|H\|_F^2.
\end{equation}
Hence, $\widehat q_t$ is $\eta_t\tau d$-strongly convex with respect to the Frobenius norm, and its constrained minimizer $D_t^\star$ is unique.

Because $D_t^\star$ minimizes $\widehat q_t$ over the convex feasible set, strong convexity implies
\begin{equation}
\widehat q_t(D)
-
\widehat q_t(D_t^\star)
\geq
\frac{\eta_t\tau d}{2}
\|D-D_t^\star\|_F^2.
\end{equation}
In particular,
\begin{equation}
\|D_{t,K}-D_t^\star\|_F^2
\leq
\frac{
2\left[
\widehat q_t(D_{t,K})
-
\widehat q_t(D_t^\star)
\right]
}{
\eta_t\tau d
}.
\end{equation}

Applying Theorem~\ref{thm:kfac_fw_inner_convergence} gives
\begin{align}
\|D_{t,K}-D_t^\star\|_F^2
&\leq
\frac{2}{\eta_t\tau d}
\frac{
8\eta_t\tau r\rho^2
\left(
\alpha_t\|A_t\|_2\|B_t\|_2+d
\right)
}{
K+2
}
\\
&=
\frac{
16r\rho^2
\left(
\alpha_t\|A_t\|_2\|B_t\|_2+d
\right)
}{
d(K+2)
}.
\end{align}
Taking square roots yields
\begin{equation}
\boxed{
\|D_{t,K}-D_t^\star\|_F
\leq
4\rho
\sqrt{
\frac{
r\left(
\alpha_t\|A_t\|_2\|B_t\|_2+d
\right)
}{
d(K+2)
}
}.
}
\label{eq:fw_iterate_convergence}
\end{equation}

\clearpage
\section{Monotonic surrogate descent}
\label{app:proof_fw_monotonic}

\begin{proposition}[Monotonic surrogate descent]
\label{prop:fw_monotonic}
Suppose
\begin{equation}
\gamma_s
\in
\arg\min_{\gamma\in[0,1]}
\widehat q_\eta
\left(
D_s+\gamma(S_s-D_s)
\right).
\label{eq:theory_exact_line_search}
\end{equation}
Then
\begin{equation}
\widehat q_\eta(D_{s+1})
\leq
\widehat q_\eta(D_s),
\label{eq:fw_monotonic}
\end{equation}
with strict inequality whenever $\mathcal G_s>0$.
\end{proposition}

Because $\gamma=0$ is feasible, exact line search cannot increase the surrogate. When $\mathcal G_s>0$, the Frank--Wolfe direction is strictly descending.

\begin{proof}
Let
\begin{equation}
\Delta_s
=
S_s-D_s,
\end{equation}
and define the one-dimensional function
\begin{equation}
\phi_s(\gamma)
=
\widehat q_\eta(D_s+\gamma\Delta_s),
\qquad
\gamma\in[0,1].
\end{equation}
Since $\gamma=0$ is feasible for the line search and $\gamma_s$ minimizes $\phi_s$ over $[0,1]$,
\begin{equation}
\widehat q_\eta(D_{s+1})
=
\phi_s(\gamma_s)
\leq
\phi_s(0)
=
\widehat q_\eta(D_s).
\end{equation}

For strict descent, the directional derivative at $\gamma=0$ is
\begin{align}
\phi_s'(0)
&=
\langle R_s,S_s-D_s\rangle
\\
&=
-\mathcal G_s.
\end{align}
If $\mathcal G_s>0$, then $\phi_s'(0)<0$. By differentiability, there exists a sufficiently small $\gamma>0$ such that
\begin{equation}
\phi_s(\gamma)
<
\phi_s(0).
\end{equation}
Exact line search therefore satisfies
\begin{equation}
\widehat q_\eta(D_{s+1})
<
\widehat q_\eta(D_s).
\end{equation}
\end{proof}

\subsection{Relation to the Muon Update}
\label{sec:theory_relation_muon}

The zero-initialized solver can be compared directly with the Muon update.

\begin{proposition}[One-step surrogate dominance over Muon]
\label{prop:fw_dominates_muon}
Suppose that $D_0=0$ and exact line search is used. Let
\begin{equation}
D_{\mathrm M}
=
-\rho\,\operatorname{msgn}(G).
\end{equation}
Then
\begin{equation}
\widehat q_\eta(D_1)
\leq
\widehat q_\eta(D_{\mathrm M}).
\label{eq:fw_one_step_dominates_muon}
\end{equation}
Moreover, for every $K\geq1$,
\begin{equation}
\widehat q_\eta(D_K)
\leq
\widehat q_\eta(D_{\mathrm M}).
\label{eq:fw_dominates_muon}
\end{equation}
\end{proposition}

Since $D_0=0$, the first Frank--Wolfe atom is $S_0=D_{\mathrm M}$. Exact line search therefore chooses the best point on the segment
\[
\{\gamma D_{\mathrm M}:\gamma\in[0,1]\}.
\]
which contains the full Muon update at $\gamma=1$. Subsequent Frank--Wolfe steps preserve the inequality by Proposition~\ref{prop:fw_monotonic}.

This comparison concerns the quadratic surrogate rather than the exact training loss, and uses zero initialization only to make the relation to Muon explicit.

\begin{proof}
Under the zero initialization $D_0=0$,
\begin{equation}
R_0
=
\nabla_D\widehat q_\eta(0)
=
G.
\end{equation}
The first Frank--Wolfe atom is therefore
\begin{equation}
S_0
=
-\rho\,\operatorname{msgn}(G)
=
D_{\mathrm M}.
\end{equation}

The first iterate has the form
\begin{equation}
D_1
=
D_0+\gamma_0(S_0-D_0)
=
\gamma_0D_{\mathrm M},
\end{equation}
where exact line search chooses
\begin{equation}
\gamma_0
\in
\arg\min_{\gamma\in[0,1]}
\widehat q_\eta(\gamma D_{\mathrm M}).
\end{equation}
Since $\gamma=1$ is feasible for this one-dimensional optimization,
\begin{equation}
\widehat q_\eta(D_1)
=
\min_{\gamma\in[0,1]}
\widehat q_\eta(\gamma D_{\mathrm M})
\leq
\widehat q_\eta(D_{\mathrm M}).
\end{equation}
This proves the one-step comparison.

For every subsequent Frank--Wolfe refinement, Proposition~\ref{prop:fw_monotonic} gives
\begin{equation}
\widehat q_\eta(D_K)
\leq
\widehat q_\eta(D_{K-1})
\leq
\cdots
\leq
\widehat q_\eta(D_1).
\end{equation}
Combining the two inequalities yields
\begin{equation}
\widehat q_\eta(D_K)
\leq
\widehat q_\eta(D_{\mathrm M}),
\qquad
K\geq1.
\end{equation}
\end{proof}

\clearpage
\section{Effect of Curvature Inflation}
\label{sec:curvature_inflation}

We next examine the curvature inflation used in Section~\ref{sec:practical_qsd}. Consider the aligned setting, and let $g_i\geq0$ and $\kappa_i\geq0$ denote the gradient magnitude and curvature of the $i$-th mode.

\begin{proposition}[Effect of curvature inflation]
\label{prop:curvature_inflation}
With calibration $\alpha>0$, inflation $\tau\geq1$, and damping
$d\geq0$, the optimal amplitude of the $i$-th mode is
\begin{equation}
s_i^\star(\tau)
=
\min\left\{
\rho,\,
\frac{g_i}
{\eta\tau(\alpha\kappa_i+d)}
\right\}.
\label{eq:inflated_mode_solution}
\end{equation}
For $g_i>0$, define the effective curvature-to-gradient ratio
\begin{equation}
\chi_i
=
\frac{\eta\rho(\alpha\kappa_i+d)}{g_i}.
\label{eq:curvature_gradient_ratio}
\end{equation}
The mode remains saturated at the spectral-norm boundary if $\tau\chi_i\leq1$, and becomes curvature-limited if $\tau\chi_i>1$. Equivalently, the transition occurs at the critical curvature
\begin{equation}
\kappa_i^{\mathrm{crit}}(\tau)
=
\frac{1}{\alpha}
\left(
\frac{g_i}{\eta\tau\rho}-d
\right),
\label{eq:critical_curvature}
\end{equation}
which decreases monotonically with $\tau$:
\begin{equation}
\frac{\partial\kappa_i^{\mathrm{crit}}}
{\partial\tau}
=
-
\frac{g_i}
{\alpha\eta\rho\tau^2}
<0.
\label{eq:critical_curvature_derivative}
\end{equation}
Once a mode is curvature-limited, its amplitude decreases as
$1/\tau$.
\end{proposition}
Intuitively, high-curvature directions are more sensitive to finite updates, since the same update magnitude induces a larger second-order change in the objective. Increasing $\tau$ therefore makes these sensitive directions more conservative while preserving less curvature-limited directions.

$Proof.$ 
\begin{equation}
G
=
U\operatorname{diag}(g_1,\ldots,g_r)V^\top,
\qquad
g_i\geq0,
\end{equation}
and suppose that the curvature operator shares the same matrix modes,
\begin{equation}
\mathcal C[u_i v_i^\top]
=
\kappa_i u_i v_i^\top,
\qquad
\kappa_i\geq0.
\end{equation}
After calibration, damping, and curvature inflation, the effective curvature operator is
\begin{equation}
\widetilde{\mathcal C}[D]
=
\tau\left(
\alpha\mathcal C[D]+dD
\right).
\label{eq:app_inflated_operator}
\end{equation}

An aligned descent direction can be written as
\begin{equation}
D
=
-
U\operatorname{diag}(s_1,\ldots,s_r)V^\top,
\qquad
0\leq s_i\leq\rho.
\label{eq:app_aligned_direction}
\end{equation}
Because the singular vectors are orthonormal, $\|D\|_2\leq\rho$ is equivalent to $0\leq s_i\leq\rho$ for every mode.

Substituting Eq.~\ref{eq:app_aligned_direction} into the quadratic surrogate gives
\begin{align}
q_\eta(D)
&=
\langle G,D\rangle
+
\frac{\eta}{2}
\langle D,\widetilde{\mathcal C}[D]\rangle
\\
&=
\sum_{i=1}^{r}
\left[
-g_i s_i
+
\frac{\eta\tau}{2}
(\alpha\kappa_i+d)s_i^2
\right].
\label{eq:app_inflation_separable}
\end{align}
Hence the optimization separates across modes. For each $i$, we solve
\begin{equation}
\min_{0\leq s_i\leq\rho}
\left\{
-g_i s_i
+
\frac{\eta\tau}{2}
(\alpha\kappa_i+d)s_i^2
\right\}.
\label{eq:app_inflation_scalar}
\end{equation}

When $\alpha\kappa_i+d>0$, the unconstrained minimizer satisfies
\begin{equation}
-g_i
+
\eta\tau(\alpha\kappa_i+d)s_i
=
0,
\end{equation}
and is therefore
\begin{equation}
\bar s_i
=
\frac{g_i}
{\eta\tau(\alpha\kappa_i+d)}.
\label{eq:app_unconstrained_amplitude}
\end{equation}
Restricting this minimizer to the feasible interval $[0,\rho]$ yields
\begin{equation}
s_i^\star(\tau)
=
\min\left\{
\rho,\,
\frac{g_i}
{\eta\tau(\alpha\kappa_i+d)}
\right\},
\end{equation}
which proves Eq.~\ref{eq:inflated_mode_solution}. If $\alpha\kappa_i+d=0$ and $g_i>0$, the scalar objective is linear and is minimized at $s_i^\star=\rho$, consistent with the limiting form above.

For $g_i>0$, the mode remains saturated at the boundary precisely when
\begin{equation}
\frac{g_i}
{\eta\tau(\alpha\kappa_i+d)}
\geq\rho.
\end{equation}
Rearranging gives
\begin{equation}
\tau
\frac{\eta\rho(\alpha\kappa_i+d)}
{g_i}
\leq1,
\end{equation}
or equivalently
\begin{equation}
\tau\chi_i\leq1,
\end{equation}
where $\chi_i$ is defined in Eq.~\ref{eq:curvature_gradient_ratio}. Hence $s_i^\star=\rho$ throughout this regime, so changes in $\tau$ do not affect the amplitude as long as the boundary condition remains satisfied.

The same condition can be expressed as a threshold on curvature:
\begin{equation}
\kappa_i
\leq
\frac{1}{\alpha}
\left(
\frac{g_i}{\eta\tau\rho}-d
\right)
=
\kappa_i^{\mathrm{crit}}(\tau).
\end{equation}
Differentiating with respect to $\tau$ gives
\begin{equation}
\frac{\partial\kappa_i^{\mathrm{crit}}}
{\partial\tau}
=
-
\frac{g_i}
{\alpha\eta\rho\tau^2}
<0.
\end{equation}
Thus increasing $\tau$ monotonically lowers the maximum curvature for which a mode can remain boundary-saturated. If $\kappa_i^{\mathrm{crit}}(\tau)<0$, no mode with $\kappa_i\geq0$ and the corresponding $g_i$ can satisfy the boundary condition.

Once the mode lies in the curvature-limited regime,
\begin{equation}
s_i^\star(\tau)
=
\frac{g_i}
{\eta\tau(\alpha\kappa_i+d)}
=
\frac{1}{\tau}
\frac{g_i}
{\eta(\alpha\kappa_i+d)}.
\label{eq:app_interior_scaling}
\end{equation}
Therefore,
\begin{equation}
\frac{\partial s_i^\star}
{\partial\tau}
=
-
\frac{g_i}
{\eta\tau^2(\alpha\kappa_i+d)}
<0
\end{equation}
whenever $g_i>0$, so inflation monotonically attenuates every curvature-limited mode.

Consequently, curvature inflation has two effects: it leaves boundary-saturated modes unchanged, while lowering the transition threshold at which modes become controlled by the quadratic penalty. For fixed $g_i$, modes with larger $\kappa_i$ cross this threshold earlier as $\tau$ increases. This completes the proof.

\clearpage
\section{Ablation for the Baseline}
\label{sec:abla}

\subsection{Isolating the Effects of QSD}

QSD differs from Muon through two coupled effects: singular-direction adaptation and a nonuniform singular spectrum. To isolate the former, we flatten the spectrum of the full QSD update, $D_{\mathrm{QSD}}=U_Q\Sigma_QV_Q^\top \rightarrow D_{\mathrm{QSD\text{-}Orth}}=\rho U_QV_Q^\top$. Thus, Muon vs. QSD-Orth evaluates the QSD-selected directions under a flat spectrum, while QSD-Orth vs. QSD captures the additional effect of spectral shaping given those directions. As shown in Table~\ref{tab:qsd_mechanism_ablation}, direction adaptation improves validation loss by $0.0048$, while restoring the adaptive spectrum yields a further $0.0105$ improvement, for a total gain of $0.0153$ over Muon. Since directions and singular values are jointly optimized, these comparisons are controlled rather than strictly additive ablations.

\begin{table}[H]
\centering
\caption{Mechanism ablations of QSD.}
\vspace{-0.9em}
\label{tab:qsd_mechanism_ablation}

\small
\setlength{\tabcolsep}{10pt}
\renewcommand{\arraystretch}{1.08}

\begin{tabular}{lccc}
\toprule
Method & Muon &  QSD-Orth & QSD   \\
\midrule
Val. Loss $\downarrow$  & 3.3312 & 3.3264(\textcolor{red}{-0.0048})  & \textbf{3.3159}(\textcolor{red}{-0.0153}) \\
\bottomrule
\end{tabular}
\end{table}
\vspace{-0.5em}

\subsection{Ablation for the baseline hyperparameters}

We tune the QSD-specific hyperparameters on GPT-124M and reuse the resulting configuration for larger-scale experiments unless otherwise stated. Tables~\ref{tab:fwstep} and~\ref{tab:kfac-dampling} show the sensitivity to the number of FW steps and K-FAC damping, respectively. Increasing the number of FW steps improves validation loss up to about three steps, after which the gains become marginal. For K-FAC damping, a moderate damping value of $10^{-6}$ performs best, while both larger and smaller damping slightly degrade validation loss.

\begin{table}[H]
\centering

\begin{minipage}[t]{0.56\textwidth}
\centering
\captionof{table}{Effect of FW Steps (Warm Start).}
\vspace{-0.9em}
\label{tab:fwstep}
\resizebox{\linewidth}{!}{
\begin{tabular}{lccccc}
\toprule
FW step & 1 & 2 & 3 & 4 & 5 \\
\midrule
Val. loss & 3.3212 & 3.3199 & 3.3159 & 3.3161 & 3.3155 \\
\bottomrule
\end{tabular}
}
\end{minipage}
\hfill
\begin{minipage}[t]{0.38\textwidth}
\centering
\captionof{table}{Effect of K-FAC Damping.}
\vspace{-0.9em}
\label{tab:kfac-dampling}
\resizebox{\linewidth}{!}{
\begin{tabular}{lccc}
\toprule
$d$ & 1e-4 & 1e-6 & 1e-8  \\
\midrule
Val. loss & 3.3199 & 3.3159 & 3.3182  \\
\bottomrule
\end{tabular}
}
\end{minipage}

\end{table}

We further study the effects of the update interval and EMA coefficient for K-FAC factor estimation and curvature-scale calibration. Figure \ref{fig:kfac-CI-EC} and Figure \ref{fig:kfac-SC-EC} show a broad region of stable performance, where most variations are small and likely comparable to training noise. However, overly frequent or infrequent updates combined with insufficient smoothing can lead to noticeable degradation.

\begin{figure}[H]
    \centering

    \begin{subfigure}[t]{0.49\linewidth}
        \centering
        \includegraphics[width=\linewidth]{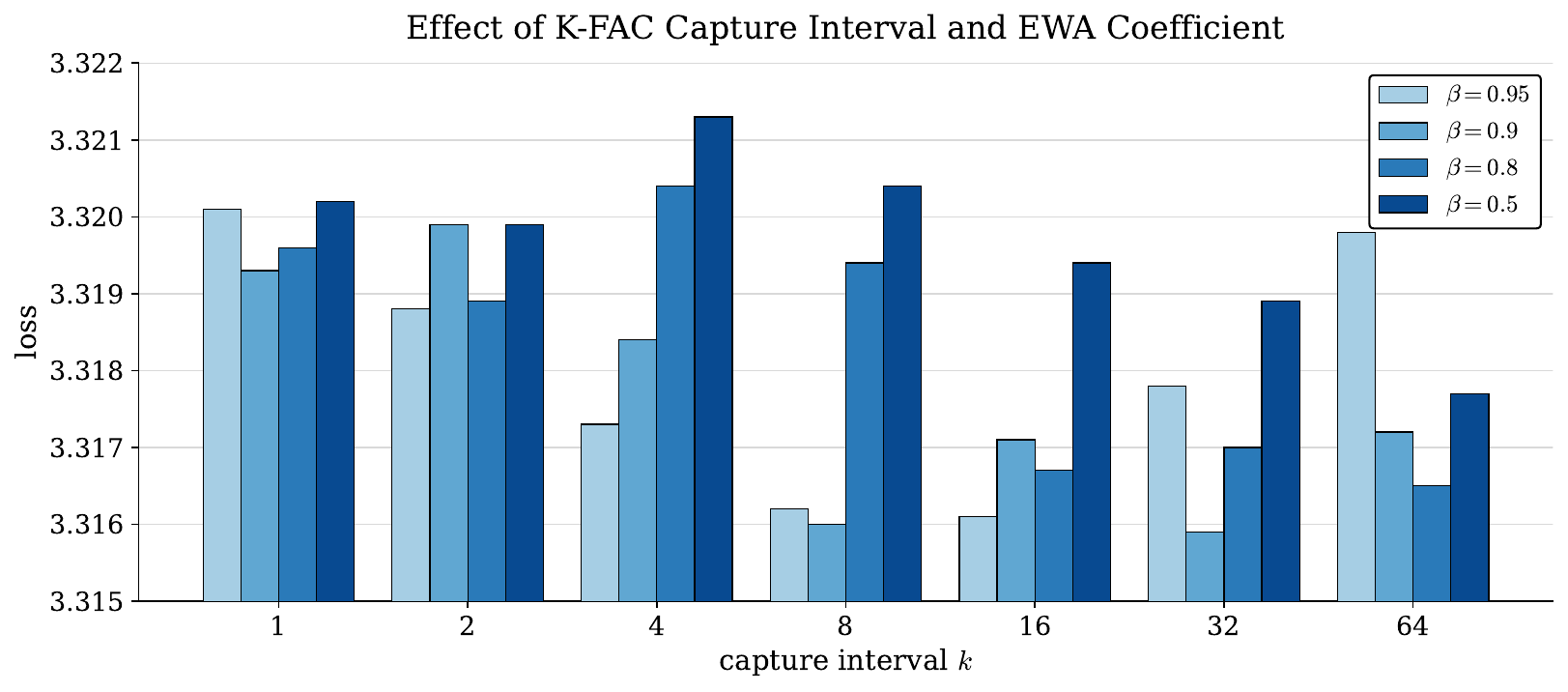}
        \caption{K-FAC Capture Interval and EMA Coefficient.}
        \label{fig:kfac-CI-EC}
    \end{subfigure}
    \hfill
    \begin{subfigure}[t]{0.49\linewidth}
        \centering
        \includegraphics[width=\linewidth]{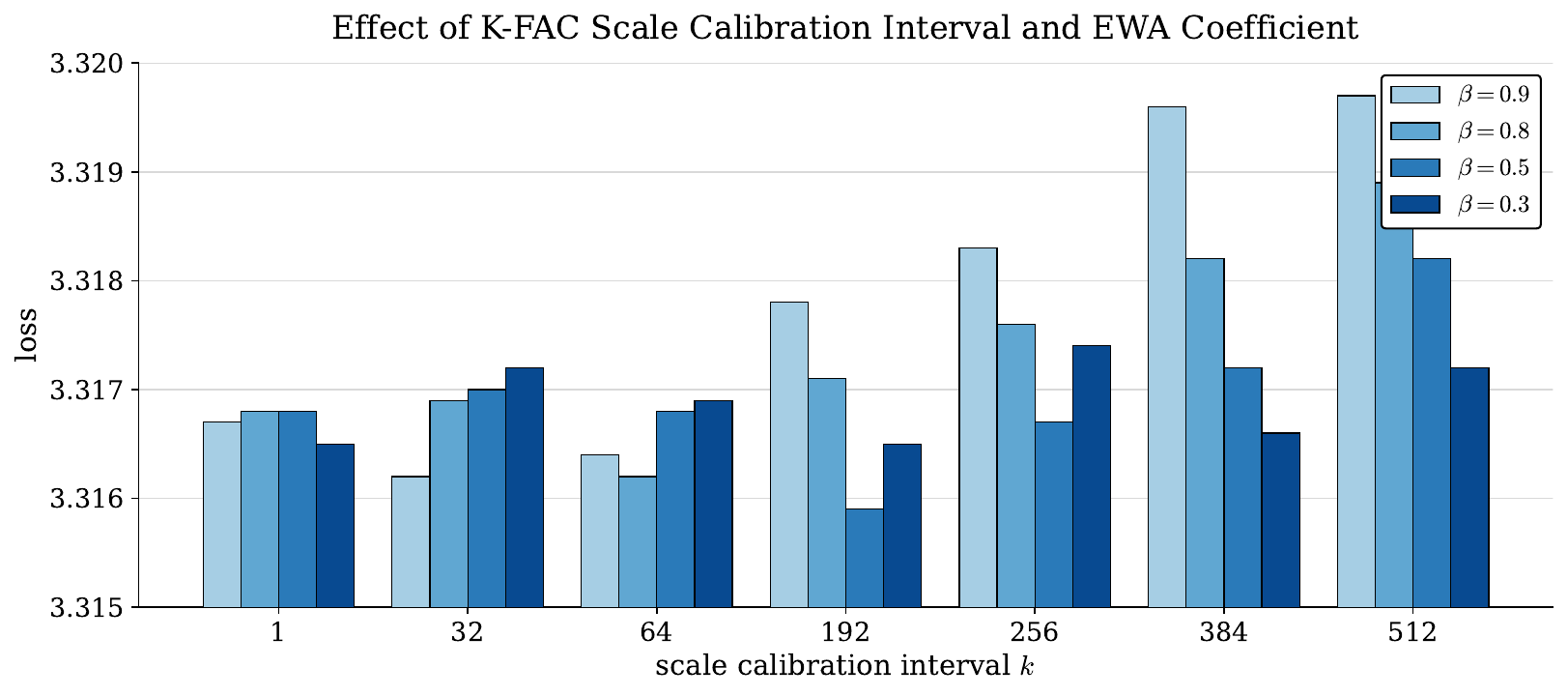}
        \caption{Scale Calibration Interval and EMA Coefficient.}
        \label{fig:kfac-SC-EC}
    \end{subfigure}

    \vspace{-0.5em}
    \caption{Sensitivity to K-FAC factor estimation and curvature-scale calibration settings.}
    \label{fig:kfac-hyperparameter-sensitivity}
\end{figure}

Tables~\ref{tab:kfac-sample-ratio} and~\ref{tab:kfac-sc-sample-ratio} examine the effects of the data sampling ratio used for K-FAC factor estimation and the sample ratio used for scale calibration. For K-FAC factor estimation, using only $1\%$ of the data achieves the best validation loss, with little benefit from larger sampling ratios, and for scale calibration, performance is stable across a wide range of sampling ratios, with $1\%$ already matching or outperforming substantially larger ratios.

\begin{minipage}[t]{0.49\textwidth}
\centering
\captionof{table}{Effect of K-FAC Sample Ratio.}
\vspace{-0.9em}
\label{tab:kfac-sample-ratio}
\resizebox{\linewidth}{!}{
\begin{tabular}{lccccc}
\toprule
Ratio & 100\% & 10\% & 1\% & 0.5\% & 0.1\% \\
\midrule
Val. loss & 3.3181 & 3.3161 & 3.3159 & 3.3173 & 3.3178 \\
\bottomrule
\end{tabular}
}
\end{minipage}
\hfill
\begin{minipage}[t]{0.49\textwidth}
\centering
\captionof{table}{K-FAC Scale Calibration Sample Ratio.}
\vspace{-0.9em}
\label{tab:kfac-sc-sample-ratio}
\resizebox{\linewidth}{!}{
\begin{tabular}{lccccc}
\toprule
Ratio & 40\% & 10\% & 1\% & 0.5\% & 0.1\% \\
\midrule
Val. loss & 3.3177 & 3.3169 & 3.3159 & 3.3169 & 3.3171 \\
\bottomrule
\end{tabular}
}
\end{minipage}

Table~\ref{tab:kfac-trs} studies the effect of the curvature inflation factor. Performance improves as $\tau$ increases up to $1.5$, suggesting that mild curvature inflation benefits training by making curvature-sensitive directions more conservative, while excessive inflation can over-regularize the update.

\begin{table}[H]
\centering
\caption{Effect of K-FAC Inflation Factor.}
\vspace{-0.9em}
\label{tab:kfac-trs}
\begin{tabular}{lcccccc}
\toprule
$\tau$ & 0.75 & 1.0 & 1.25 & 1.5 & 1.75 & 2.0  \\
\midrule
Val. loss & 3.3193 & 3.3173 & 3.3166 & 3.3159 & 3.3168 & 3.3175  \\
\bottomrule
\end{tabular}
\end{table}

We also provide the effect of the FW step under zero start. As shown in Table \ref{tab:FW-step-zs}, for the same number of FW steps, zero start consistently underperforms warm start. With 5 FW steps, zero start still performs worse than warm start with 3 steps. 

\begin{table}[H]
\centering
\caption{Effect of FW Steps (Zero Start).}
\vspace{-0.9em}
\label{tab:FW-step-zs}
\begin{tabular}{cccccc}
\toprule
FW step   & 1 & 2 & 3 & 4 & 5 \\
\midrule
Val. loss  & 3.3258 & 3.3200 & 3.3183 & 3.3192 & 3.3177 \\
\bottomrule
\end{tabular}
    
\end{table}

\clearpage
\section{K-FAC Factor Estimation Details}
\label{app:kfac_factor_estimation}

We estimate the K-FAC factors in Eqs.~\ref{eq:minibatch_kfac_factors}--\ref{eq:kfac_factor_ema} from a small subset of token positions in the current training minibatch. The same sampled positions are used for the activation and output-gradient statistics so that each pair $(a_{i,\ell},\delta_{i,\ell})$ corresponds to the same token.

\subsection{Token Sampling}

Factor statistics are sampled at the token level. For each microbatch containing $N$ token positions, we uniformly sample without replacement
\begin{equation}
n_{\mathrm{keep}}
=
\max
\left\{
1,\,
\left\lceil r_F N \right\rceil
\right\}
\end{equation}
positions, where $r_F$ is the factor-sampling ratio.

In our default setting $r_F=0.01$, so each factor estimate uses approximately $1\%$ of the $524288$ tokens in an optimizer step ($5248$ sampled positions). The per-microbatch counts depend on the per-device batch size and the number of devices; under data-parallel training the per-device sufficient statistics $\sum a a^\top$, $\sum \delta\delta^\top$ and the sample counts are summed across devices before normalization, so the pooled estimate is unaffected by the device count.

The factors are averaged over the number of sampled token positions. The same sampling indices are used for $A_{t,\ell}$ and $B_{t,\ell}$ to preserve the correspondence between activations and their output gradients.

\subsection{Output-Gradient Rescaling}

The output gradients used to construct $B_{t,\ell}$ are obtained from backward hooks. Their raw values contain the scaling introduced by mean-loss reduction, gradient accumulation, AMP loss scaling, and any additional loss scaling used during training. We restore the corresponding per-token scale before computing the K-FAC statistics.

Let $\delta^{\mathrm{hook}}_{i,\ell}$ denote the output gradient captured by the backward hook. The gradient used in the factor estimate is
\begin{equation}
\delta_{i,\ell}
=
\frac{
N\,n_{\mathrm{accum}}
}{
s_{\mathrm{AMP}}\,s_{\mathrm{custom}}
}
\,
\delta^{\mathrm{hook}}_{i,\ell},
\label{eq:kfac_delta_rescaling}
\end{equation}
where $N$ is the number of token positions in the original microbatch, $n_{\mathrm{accum}}$ is the number of accumulated microbatches, $s_{\mathrm{AMP}}$ is the AMP loss-scaling factor, and $s_{\mathrm{custom}}$ denotes any additional multiplicative loss scaling. We set $s_{\mathrm{custom}}=1$ when no additional scaling is used.

Importantly, $N$ in Eq.~\ref{eq:kfac_delta_rescaling} is the token count before K-FAC subsampling. Subsampling changes only the number of terms used to estimate the factors and does not change the scale of the underlying per-token gradients.

The output-side factor is then computed as
\begin{equation}
\widehat B_{t,\ell}
=
\frac{1}{|\mathcal S_t|}
\sum_{i\in\mathcal S_t}
\delta_{i,\ell}\delta_{i,\ell}^{\top},
\end{equation}
using the same sampled token positions as the corresponding activation-side factor.

\subsection{Factor Refresh and EMA}

The K-FAC factors are refreshed every $j=32$ optimizer steps. At a refresh step, the current estimates are incorporated using
\begin{equation}
A_{t,\ell}
=
\beta_F A_{t-j,\ell}
+
(1-\beta_F)\widehat A_{t,\ell},
\qquad
B_{t,\ell}
=
\beta_F B_{t-j,\ell}
+
(1-\beta_F)\widehat B_{t,\ell},
\end{equation}
with $\beta_F=0.9$.

At the first refresh, the running factors are initialized directly from the first estimates,
\begin{equation}
A_{t,\ell}
=
\widehat A_{t,\ell},
\qquad
B_{t,\ell}
=
\widehat B_{t,\ell}.
\label{eq:kfac_factor_initialization}
\end{equation}
Since the EMA is not initialized from zero, no bias correction is used. Between refresh steps, the most recently stored factors are reused.

\subsection{Numerical Precision}

The sampled activations and output gradients follow the training precision and are stored in bfloat16. The matrix products used to form the K-FAC statistics are therefore computed from bfloat16 operands using the corresponding tensor-core GEMM path. Accumulation across microbatches is performed in FP32, and the resulting K-FAC factors and their EMA states are stored in FP32.

\clearpage
\section{Directional Curvature Calibration}
\label{app:curvature_calibration}

K-FAC can have a scale mismatch with the corresponding generalized Gauss--Newton (GGN) curvature. We correct this mismatch with a per-layer scalar coefficient by matching the two directional curvatures along an accepted update direction.

Let $\bar D_{t,\ell}$ denote the stored update direction used for calibration for layer $\ell$, and let $\theta_t$ be the parameter point associated with this direction. Both directional curvatures below are evaluated at the same parameter point $\theta_t$.

\subsection{Directional Curvature Matching}

The K-FAC directional curvature along $\bar D_{t,\ell}$ is
\begin{align}
c^{\mathrm{KFAC}}_{t,\ell}
&=
\left\langle
\bar D_{t,\ell},
B_{t,\ell}\bar D_{t,\ell}A_{t,\ell}
\right\rangle
\nonumber\\
&=
\operatorname{tr}
\left(
\bar D_{t,\ell}^{\top}
B_{t,\ell}
\bar D_{t,\ell}
A_{t,\ell}
\right).
\label{eq:kfac_directional_curvature}
\end{align}
Damping is excluded from this calculation. Thus, $\alpha_{t,\ell}$ calibrates only the Kronecker-factored term $B_{t,\ell}DA_{t,\ell}$, while the damping term $dD$ and the inflation factor $\tau$ are applied separately.

For the GGN curvature, let $z_i(\theta)$ denote the logits for token $i$, and let $J_{i,t,\ell}$ be the Jacobian of $z_i$ with respect to the parameters of layer $\ell$ at $\theta_t$. For softmax cross-entropy,
\begin{equation}
H_{i,t}
=
\operatorname{diag}(p_{i,t})
-
p_{i,t}p_{i,t}^{\top},
\qquad
p_{i,t}
=
\operatorname{softmax}
\left(
z_i(\theta_t)
\right),
\label{eq:logit_hessian}
\end{equation}
is the Hessian with respect to the logits. Define
\begin{equation}
v_{i,t,\ell}
=
J_{i,t,\ell}
\operatorname{vec}
\left(
\bar D_{t,\ell}
\right).
\label{eq:calibration_logit_tangent}
\end{equation}
The corresponding GGN directional curvature is
\begin{equation}
c^{\mathrm{GGN}}_{t,\ell}
=
\frac{1}{N_t^{\mathrm{cal}}}
\sum_{i\in\mathcal T_t^{\mathrm{cal}}}
v_{i,t,\ell}^{\top}
H_{i,t}
v_{i,t,\ell},
\label{eq:ggn_directional_curvature}
\end{equation}
where $\mathcal T_t^{\mathrm{cal}}$ contains all tokens from the sequences selected for calibration and
\begin{equation}
N_t^{\mathrm{cal}}
=
|\mathcal T_t^{\mathrm{cal}}|.
\end{equation}
The GGN estimate and the K-FAC statistics use the same per-token averaging convention.

The instantaneous calibration coefficient is
\begin{equation}
\widehat\alpha_{t,\ell}
=
\frac{
c^{\mathrm{GGN}}_{t,\ell}
}{
c^{\mathrm{KFAC}}_{t,\ell}
}.
\label{eq:instantaneous_curvature_calibration}
\end{equation}
Hence, before clipping and temporal smoothing, $\widehat\alpha_{t,\ell}$ matches the GGN curvature along $\bar D_{t,\ell}$ without changing the Kronecker structure of $A_{t,\ell}$ and $B_{t,\ell}$.

Layers with
\begin{equation}
c^{\mathrm{KFAC}}_{t,\ell}
\leq
\epsilon_{\mathrm{skip}}
\end{equation}
are skipped for that calibration step and retain their previous calibration coefficient, where
\begin{equation}
\epsilon_{\mathrm{skip}}
=
10^{-20}.
\end{equation}

The instantaneous ratio is clipped before the EMA update:
\begin{equation}
\alpha_{t,\ell}
=
\beta_{\alpha}\alpha_{t-k,\ell}
+
(1-\beta_{\alpha})
\operatorname{clip}
\left(
\widehat\alpha_{t,\ell},
\alpha_{\min},
\alpha_{\max}
\right).
\label{eq:app_calibration_ema}
\end{equation}
Between calibration steps, the most recent $\alpha_{t,\ell}$ is reused.

\subsection{Finite-Difference GGN Estimate}

We do not form the GGN matrix or the Jacobian explicitly. Only the Jacobian--vector product in Eq.~\ref{eq:calibration_logit_tangent} is needed, and we approximate it using a forward finite difference.

We first normalize the calibration direction,
\begin{equation}
\widetilde D_{t,\ell}
=
\frac{
\bar D_{t,\ell}
}{
\|\bar D_{t,\ell}\|_F
}.
\label{eq:normalized_calibration_direction}
\end{equation}
Let $\mathcal E_\ell[D]$ denote a perturbation that changes only layer $\ell$ by $D$. Then
\begin{align}
v_{i,t,\ell}
&=
J_{i,t,\ell}
\operatorname{vec}
\left(
\bar D_{t,\ell}
\right)
\nonumber\\
&\approx
\frac{
\|\bar D_{t,\ell}\|_F
}{
\delta
}
\left[
z_i
\left(
\theta_t
+
\delta\,
\mathcal E_\ell[
\widetilde D_{t,\ell}]
\right)
-
z_i(\theta_t)
\right].
\label{eq:finite_difference_jvp}
\end{align}

Substituting Eq.~\ref{eq:finite_difference_jvp} into Eq.~\ref{eq:ggn_directional_curvature} gives the GGN directional curvature using only forward evaluations and the closed-form logit-space Hessian in Eq.~\ref{eq:logit_hessian}. No explicit Jacobian, GGN matrix, or additional backward pass is required.

The direction and curvature are evaluated at the same parameter point. In the implementation, the parameters are restored to the $\theta_t$ associated with $\bar D_{t,\ell}$ before computing the unperturbed and perturbed logits. The unperturbed logits $z_i(\theta_t)$ are computed once and shared across the layerwise calibration calculations.

\subsection{Implementation Details}

We perform curvature calibration every
\begin{equation}
k=192
\end{equation}
training iterations and use
\begin{equation}
\beta_{\alpha}=0.5,
\qquad
\alpha_{\min}=0.05,
\qquad
\alpha_{\max}=100.
\end{equation}

Calibration data are sampled at the sequence level rather than at the token level. We nominally sample $1\%$ of the sequences in the current minibatch. A minimum calibration block size of $4$ sequences is used, and the selected sequences are processed in blocks of $4$. With the minibatch size of $512$ used in our experiments, the nominal $1\%$ setting results in $8$ sequences per calibration step, corresponding to $1.56\%$ of the minibatch. All tokens in each selected sequence are included in the directional-curvature estimate.

For the finite-difference approximation, we use
\begin{equation}
\delta=0.1
\end{equation} 
after normalizing the perturbation by $\|\bar D_{t,\ell}\|_F$. The finite-difference computation is carried out in FP32 with TF32 disabled. This is important because Eq.~\ref{eq:finite_difference_jvp} subtracts two nearby logit vectors and is sensitive to reduced mantissa precision.

\clearpage
\section{Experimental Details}
\label{app:experimental_details}

Table~\ref{tab:training_hyperparameters} summarizes the training and QSD-specific hyperparameters used in our main experiments. Unless otherwise stated, QSD-specific hyperparameters tuned on GPT-124M are kept unchanged for GPT-350M.

\begin{table}[H]
\centering
\caption{Hyperparameters used in the main experiments.}
\label{tab:training_hyperparameters}
\small
\begin{tabular}{lcc}
\toprule
Hyperparameter & GPT-124M & GPT-350M \\
\midrule

\multicolumn{3}{l}{\textit{Training setup}} \\

Dataset
& FineWeb
& FineWeb \\

Sequence length
& 1024
& 1024 \\

Batch size (sequences)
& 512
& 512 \\

Training tokens
& 2.5B
& 7B \\

Precision
& bfloat16
& bfloat16 \\

Number of GPUs
& 4
& 4 \\

GPU
& RTX 5000 Ada
& RTX 5000 Ada \\

Learning-rate schedule
& WSD
& WSD \\

Warmup steps / ratio
& 0
& 0 \\

Weight decay
& 0
& 0 \\

\midrule
\multicolumn{3}{l}{\textit{QSD optimizer}} \\

QSD momentum
& 0.95
& 0.95 \\

QSD learning rate
& 0.02
& 0.015 \\

Spectral radius $\rho$
& 1
& 1 \\

Number of FW steps $K$
& 3
& 3 \\

K-FAC damping $d$
& $10^{-6}$
& $10^{-6}$ \\

Curvature inflation $\tau$
& 1.5
& 1.5 \\

\midrule
\multicolumn{3}{l}{\textit{Muon optimizer}} \\

Muon momentum
& 0.95
& 0.95 \\

Muon learning rate
& 0.01
& 0.01 \\

Spectral radius $\rho$
& 1
& 1 \\

\midrule
\multicolumn{3}{l}{\textit{K-FAC factor estimation}} \\

Factor sampling ratio $r_F$
& $1\%$
& $1\%$ \\

Factor refresh interval $j$
& 32 optimizer steps
& 128 optimizer steps \\

Factor EMA $\beta_F$
& 0.9
& 0.8 \\

Factor storage / EMA precision
& FP32
& FP32 \\

\midrule
\multicolumn{3}{l}{\textit{Curvature-scale calibration}} \\

Calibration interval $k$
& 192 optimizer steps
& 384 optimizer steps \\

Calibration EMA $\beta_\alpha$
& 0.5
& 0.5 \\

Calibration clip
& $[0.05,100]$
& $[0.05,100]$ \\

Nominal sequence sampling ratio
& $1\%$
& $1\%$ \\

Finite-difference step $\delta$
& 0.1
& 0.1 \\

Finite-difference precision
& FP32 (TF32 disabled)
& FP32 (TF32 disabled) \\

\midrule
\multicolumn{3}{l}{\textit{AdamW parameters}} \\

AdamW learning rate
& 0.0056
& 0.0032 \\

AdamW $\beta_1$
& 0.9
& 0.9 \\

AdamW $\beta_2$
& 0.95
& 0.95 \\

\bottomrule
\end{tabular}
\end{table}

\clearpage
\section{Limitations}
\label{sec:limitations}

QSD relies on a local quadratic model and therefore depends on the quality of the curvature approximation. Our implementation uses Kronecker-factored empirical-Fisher statistics with GGN-based scale calibration, which may still miss aspects of the true curvature.

Our theoretical guarantees apply to the quadratic surrogate rather than directly to the exact training loss, since higher-order effects are not modeled.

Finally, our experiments focus on GPT pre-training at the 124M and 350M scales. Evaluating QSD on larger models, other architectures, and post-training settings remains future work.

\end{document}